\documentclass{article} %
\usepackage[main, final]{neurips_2026}

\usepackage[utf8]{inputenc} %
\usepackage[T1]{fontenc}    %
\usepackage{hyperref}       %
\usepackage{url}            %
\usepackage{booktabs}       %
\usepackage{amsfonts}       %
\usepackage{nicefrac}       %
\usepackage{microtype}      %

\usepackage[dvipsnames]{xcolor}

\usepackage{microtype}
\usepackage{graphicx}
\usepackage{subcaption}
\usepackage{booktabs} %

\usepackage{amsmath}
\usepackage{amssymb}
\usepackage{mathtools}
\usepackage{amsthm}

\usepackage{stmaryrd}

\usepackage{hyperref}
\hypersetup{
    colorlinks,
    linkcolor={blue!60!black},
    citecolor={cyan!80!black},
    urlcolor={magenta!60!black}
}

\usepackage[capitalize]{cleveref}

\usepackage{dsfont} %
\usepackage{xcolor}
\usepackage{colortbl}
\usepackage{enumitem}
\usepackage{wrapfig}
\usepackage{multirow}
\usepackage{fancybox}
\usepackage{changepage}
\usepackage{graphicx} %
\usepackage{blindtext}
\usepackage{titlesec}
\usepackage{amsthm}
\usepackage{bm}
\usepackage{booktabs}
\usepackage{array}
\usepackage{diagbox}
\usepackage{subcaption}
\usepackage{threeparttable}
\usepackage{tabularx}  
\usepackage{algorithm}
\usepackage{algorithmic}
\usepackage{pgfmath}

\usepackage{pgfplots}
\usepackage{pgfplotstable}
\usepgfplotslibrary{groupplots} %
\usepgfplotslibrary{fillbetween}
\usetikzlibrary{arrows.meta, positioning}
\pgfplotsset{compat=1.18}
\pgfplotsset{grid = major, grid style={gray!30!white}}

\definecolor{darkgreen}{rgb}{0.31, 0.47, 0.26}

\usepackage[textsize=tiny]{todonotes}

\usepackage[framemethod=TikZ]{mdframed}

\definecolor{theoremcolor}{rgb}{0.931176471, 0.958627451, 1.0}

\mdfsetup{
    backgroundcolor=theoremcolor,
    linecolor=black,      %
    linewidth=0.5pt,      %
    frametitlebackgroundcolor=theoremcolor, %
    leftline=true,
    topline=true,
    rightline=true,
    bottomline=true,
    skipabove=10pt,
    skipbelow=10pt,
    roundcorner=2pt       %
}

\newmdtheoremenv{definition}{Definition}
\newmdtheoremenv{proposition}{Proposition}
\newmdtheoremenv{corollary}{Corollary}
\newmdtheoremenv{theorem}{Theorem}
\newmdtheoremenv{lemma}{Lemma}
\newmdtheoremenv{example}{Example}
\newmdtheoremenv{remark}{Remark}

\crefname{problem}{Problem}{Problems}
\crefname{definition}{Definition}{Definitions}

\newlist{inlinelist}{enumerate*}{1}
\setlist[inlinelist]{label=(\roman*)} %

\definecolor{LightGray}{gray}{0.92}

\usepackage{amsmath,amsfonts,bm}

\def\1{\bm{1}}
\def\0{\bm{0}}

\DeclareMathOperator*{\argmax}{arg\,max}
\DeclareMathOperator*{\argmin}{arg\,min}

\def\RR{{\mathbb{R}}}
\def\NN{{\mathbb{N}}}
\def\EE{{\mathbb{E}}}
\def\PP{{\mathbb{P}}}

\def\conv{{\mathrm{conv}}}
\def\range{{\mathrm{range}}}

\def\cC{{\mathcal{C}}}
\def\cD{{\mathcal{D}}}
\def\cE{{\mathcal{E}}}

\def\cG{{\mathcal{G}}}
\def\cH{{\mathcal{H}}}

\def\cL{{\mathcal{L}}}
\def\cM{{\mathcal{M}}}
\def\cN{{\mathcal{N}}}
\def\cO{{\mathcal{O}}}
\def\cP{{\mathcal{P}}}

\def\cV{{\mathcal{V}}}

\def\cY{{\mathcal{Y}}}
\def\cZ{{\mathcal{Z}}}

\def\thetav{{\bm{\theta}}}
\def\muv{{\bm{\mu}}}

\def\b{{\bm{b}}}

\def\e{{\bm{e}}}

\def\p{{\bm{p}}}
\def\q{{\bm{q}}}
\def\r{{\bm{r}}}
\def\s{{\bm{s}}}

\def\u{{\bm{u}}}
\def\v{{\bm{v}}}

\def\x{{\bm{x}}}
\def\y{{\bm{y}}}
\def\z{{\bm{z}}}
\def\A{{\bm{A}}}

\def\I{{\bm{I}}}

\def\maxo{f_\Omega}
\def\argmaxo{\p_\Omega}

\def\triangleY{\triangle^{\mathcal{Y}}}
\def\RRY{\mathbb{R}^{|\mathcal{Y}|}}

\def\TV{d_{\mathrm{TV}}}

\def\KL{D_{\mathrm{KL}}}

\def\qh{\hat{\bm{q}}}

\def\piv{\bm{\pi}}

\def\gap{\delta_{\bm{\thetav},\Omega}}
\def\gape{\delta_{\bm{\thetav},\Omegae}}

\def\qlns{\bm{q}^{\mathrm{LNS}}(\bm{y}^{(k)}\!\!,S^{(k)})}
\def\qgibbs{\bm{q}^{\mathrm{Gibbs}}(\bm{y}^{(k)}\!\!,S^{(k)})}

\def\st{\bm{s}_{\bm{\thetav}}}

\def\supp{\mathrm{supp}}

\def\Omegae{\Omega_{\varepsilon}}

\def\dom{\mathrm{dom}}
\def\supp{\mathrm{supp}}

\def\ri{\mathrm{ri}}

\def\KL{D_{\mathrm{KL}}}

\def\piLNS{\pi^{\mathrm{LNS}}_{\bm{\theta},\Omega}}
\def\piLNSe{\pi^{\mathrm{LNS}}_{\bm{\theta},\Omegae}}

\def\Sk{S^{(k)}}
\def\Fk{F^{(k)}}
\def\Mk{\cM^{(k)}}
\def\yk{\y^{(k)}}
\def\Yk{\cY^{(k)}}
\def\Omegak{\Omega^{\Yk}}

\def\argmaxok{\p_{\Omegak}}
\def\stk{(\bm{s}_{\bm{\thetav}})_{[\Yk]}}
\def\trianglek{\triangle^{\Yk}}
\def\iotak{\iota_{\Yk}}

\def\argmaxoe{\p_{\Omega_\varepsilon}}
\def\argmaxoek{\p_{\Omega_\varepsilon^{\cY^{(k)}}}}

\begin{document}

\title{Regularized Large Neighborhood Search}

\author{%
  Germain Vivier-Ardisson\\
  Google DeepMind \& CERMICS\\
  Paris, France \\
  \texttt{gvivier@google.com} \\
  \And
  Laurent Demonet \\
  Google Research \\
  Paris, France \\
  \texttt{demonet@google.com} \\
  \AND
  Axel Parmentier \\
  CERMICS, ENPC, CNRS, IPP\\
  Marne-la-Vallée, France \\
  \texttt{axel.parmentier@enpc.fr}
  \And
  Mathieu Blondel \\
  Google DeepMind \\
  Paris, France \\
  \texttt{mblondel@google.com}
}

\maketitle

\begin{abstract}
Operations research practitioners typically tackle NP-hard combinatorial problems using large neighborhood search (LNS), a scalable heuristic that iteratively refines a current solution by locally re-optimizing subsets of its variables.
In contrast, most existing approaches for integrating combinatorial optimization layers into neural networks still assume access to an exact global solution, which is computationally intractable. We bridge this gap by introducing regularized LNS (RLNS). By regularizing or perturbing local subproblems, we turn the LNS heuristic into an efficient MCMC sampler over the combinatorial set of feasible solutions, with associated Fenchel-Young losses.
Under entropic regularization, we prove that RLNS performs exact block Gibbs sampling. 
Furthermore, adjusting the number of RLNS iterations allows us to interpolate between pseudolikelihood and exact maximum likelihood estimation, for end-to-end learning without global solvers.
We demonstrate our approach on $k$-subset selection, generalized assignment, and stochastic vehicle scheduling problems.

\end{abstract}

\section{Introduction}
\label{sec:intro}

Integrating combinatorial optimization solvers into neural networks has emerged as a powerful paradigm for structured prediction \citep{dalle_learning_2022,mandi_decision-focused_2024,sadana_survey_2025,schiffer_combinatorial_2026}. Such architectures enable machine learning models to make discrete decisions that satisfy complex constraints, via output layers of the form
\begin{align}
\label{pb:ilp}
    \thetav \longmapsto  \hat{\y}(\thetav)\triangleq \argmax_{\y \in \cY} \langle \thetav, \y \rangle,
\end{align}
where $\cY$ is a finite, but combinatorially large, set of feasible outputs. The fundamental challenge in this setting is differentiability. Because the $\argmax$ operator provides zero gradients almost everywhere, standard approaches typically rely on introducing a regularization term to yield a continuous, differentiable relaxation \citep{niculae_sparsemap_2018, mandi_interior_2020, berthet_learning_2020}. As summarized in \cref{tab:overview}, various methods have been proposed to solve the resulting problem, depending on the type of regularization used and available oracle. However, most of the existing literature fundamentally assumes access to a \emph{global} exact solver for Problem~(\ref{pb:ilp}) (a global MAP oracle) or marginal oracle \citep{lafferty_conditional_2001,wainwright_graphical_2008}, which computes the expectation of the Gibbs distribution $\p(\y)\propto\exp(\langle\thetav,\y\rangle)$ over the entire combinatorial set $\cY$. For NP-hard optimization problems encountered in operations research, calling a global oracle at every training step can be computationally intractable.

Instead, practitioners rely on highly scalable heuristics, particularly large neighborhood search (LNS). LNS iteratively improves a feasible solution by freezing a subset of its variables and exactly re-optimizing the remaining free variables~\citep{pisinger2018large}, thereby relying on \emph{local} MAP oracles only. Our goal is not to improve LNS as an optimization heuristic, but to leverage these local oracles to enable learning in combinatorial spaces where global solvers fail. To bridge this gap, we frame gradient estimation as a sampling problem and regularize LNS into a rejection-free MCMC sampler over the feasible set. While recent work connected local search heuristics with Metropolis-Hastings \citep{vivier-ardisson_learning_2025}, our approach connects LNS to Gibbs sampling. 
Specifically, we make the following contributions:
\begin{itemize}[topsep=0pt,itemsep=1.5pt,parsep=2pt,leftmargin=15pt]
    \item In \cref{sec:framework}, we introduce regularized LNS (RLNS), a stochastic version of LNS based on smoothed maximum operators, integrating regularization into the local subproblems.
    \item In \cref{sec:entropy_lns}, we prove that under entropic regularization, RLNS exactly implements a block Gibbs sampler whose stationary distribution matches a globally regularized Gibbs distribution.
    \item In \cref{sec:perturbed_lns}, we show that RLNS supports perturbation-based regularization \citep{berthet_learning_2020}, leading to tractable updates using the standard local MAP oracles of LNS heuristics.
    \item In \cref{sec:lns_fyl}, we demonstrate that the number of RLNS iterations allows us to interpolate between pseudolikelihood \citep{besag_spatial_1974} and exact maximum likelihood estimation.
    \item In \cref{sec:experiments}, we evaluate RLNS on a controlled estimation task on $\kappa$-hot vectors, a generalized assignment problem learning task, and a stochastic vehicle scheduling problem.
\end{itemize}
A summary of notations is available at the beginning of the appendix.

\begin{table*}[t]
\centering
\caption{
The new connection we establish between LNS and Gibbs sampling enables us to leverage local oracles and effectively regularize the problem when usual global oracles are not available.}
\begin{small}
\setlength{\tabcolsep}{4.3pt}
\begin{tabular}{cccc}
\toprule
& \textbf{Regularization} & \textbf{Oracle} & \textbf{Approach} \\
\midrule
Differentiable DP (\citeyear{mensch_differentiable_2018}) & Entropic & Global marginal & DP \\
SparseMAP (\citeyear{niculae_sparsemap_2018}) & Quadratic & Global MAP & Frank-Wolfe \\
Barrier FW (\citeyear{krishnan_barrier_2015}) & TRW Entropy & Global MAP & Frank-Wolfe \\
IntOpt (\citeyear{mandi_interior_2020}) & Logarithmic barrier & Interior point solver & Primal-dual \\
Perturbed optimizers (\citeyear{berthet_learning_2020}) & Implicit via noise & Global MAP & Monte-Carlo \\
DYS-net (\citeyear{mckenzie_differentiating_2024}) & Quadratic & Projection oracles & Davis-Yin splitting \\
Blackbox solvers (\citeyear{vlastelica_differentiation_2020}) & None & Global MAP & Interpolation \\
Local search MCMC (\citeyear{vivier-ardisson_learning_2025}) & Entropic & Local search & Metropolis-Hastings \\
\midrule
\rowcolor{LightGray}
RLNS (proposed) & Entropic / Implicit via noise & \emph{Local} marginal / MAP & Gibbs sampling \\
\bottomrule
\end{tabular}
\end{small}
\label{tab:overview}
\vspace{-0.4cm}
\end{table*}

\section{Background and related work}
\label{sec:background_related_work}

\paragraph{Large neighborhood search.} LNS is a meta-heuristic used to approximately solve Problem~(\ref{pb:ilp}) when it is intractable (e.g., when it is NP-hard), by using a \emph{local} solver (\emph{local} MAP oracle):
\begin{align}
\thetav\mapsto \argmax_{\y'\in\cY}\;\langle\thetav,\y'\rangle\quad\text{s.t.}\quad \y'_{[S]}=\y_{[S]}.
\label{eq:local_map}
\end{align}
Such an oracle exactly solves sub-problems of (\ref{pb:ilp}) by fixing a subset $S\subseteq[d]$ of variables to the previous solution $\y \in \cY$. LNS iteratively refines a solution by freezing different variables at each step and exactly re-optimizing the rest \citep{archetti2014survey,pisinger2018large}. The pseudo-code is given in \cref{algo:lns}.

\begin{algorithm}[H]
\caption{Large neighborhood search (LNS)}
\label{algo:lns}
\begin{algorithmic}[1]
\STATE\textbf{Inputs:}  Logits $\thetav\in\RR^d$\!, feasible set $\cY \!\subset \!\RR^d$\!, initial solution $\y^{(0)\!} \!\in \!\cY$, number of iterations $K\!\in\!\NN$.
\FOR{$k = 0$ to $K-1$}
    \STATE Choose a subset of frozen coordinates $\Sk \subset [d]$.
    \STATE Define the local neighborhood: $\Yk\triangleq \left\{\y \in \cY \mid \y_{[\Sk]}=\yk_{[\Sk]}\right\}$
    \STATE Set $\y^{(k+1)}\gets \argmax_{\y\in\Yk}\langle\thetav,\y\rangle$ \COMMENT{Solve the local optimization subproblem.}
\ENDFOR
\STATE \textbf{Return $\y^{(K)}$}
\end{algorithmic}
\end{algorithm}

\paragraph{Smoothed maximum operators.}

For a finite set $\cZ$, let $\Omega:\triangle^\cZ\to\RR$ be a lower semicontinuous (l.s.c.) strictly convex regularizer. The smoothed maximum operator $\maxo:\RR^{|\cZ|}\to\RR$ \citep{mensch_differentiable_2018} and smoothed maximizer $\argmaxo:\RR^{|\cZ|}\to\triangle^\cZ$ are defined as:
\begin{align*}
\maxo(\s)\triangleq\max_{\q\in\triangle^\cZ}\,\langle\s,\q\rangle-\Omega(\q)=\Omega^*(\s),\quad\text{and}\quad \argmaxo(\s)\triangleq\argmax_{\q\in\triangle^\cZ}\,\langle\s,\q\rangle-\Omega(\q)=\nabla\Omega^*(\s),
\end{align*}
where $\cdot^*$ denotes the convex conjugate and the rightmost equality is given by Danskin's theorem \citep{danskin_theory_1966}. When $\Omega$ is \emph{Legendre-type} \citep{rockafellar_convex_1970}, i.e., strictly convex and differentiable on $\ri(\triangle^\cZ)$ with gradient norm diverging at the relative boundary of $\triangle^\cZ$, we have $\range(\argmaxo)=\ri(\triangle^\cZ)$. Typical choices for $\Omega$ include Shannon's negative entropy, yielding $\maxo(\s)=\log\sum_{\z\in\cZ}\exp(s_\z)$ and the \emph{softmax} $\argmaxo(\s)_\z=\exp(s_\z)/\sum_{\z'}\exp(s_{\z'})$, or the squared Euclidean norm $\frac{1}{2}\|\cdot\|_2^2$, giving the \emph{sparsemax} operator instead \citep{martins_softmax_2016}. In \cref{sec:perturbed_lns}, we define $\Omega$ via low-rank dual additive noise (\cref{def:perturbation_regularization}). Like Shannon's negative entropy (and unlike $\frac{1}{2}\|\cdot\|_2^2$), it is Legendre-type when $\cZ\subset\RR^d$ consists of extreme points of $\conv(\cZ)$ \citep{bouvier_primal-dual_2025}.

\paragraph{Global Fenchel-Young learning.}

Following \citet{blondel_learning_2020}, Problem~(\ref{pb:ilp}) can be lifted over the distribution simplex $\triangleY$ using the \emph{score vector} $\st\triangleq(\langle\thetav,\y\rangle)_{\y\in\cY}\in\RRY$. This enables adding distribution-space regularization into the problem via a \emph{global} regularizer $\Omega:\triangleY\to\RR$.

\begin{definition}[Distribution-space Fenchel-Young loss]
\label{def:dist_fyl}
    For a l.s.c. strictly convex regularizer $\Omega:\triangleY\to\RR$, the \emph{distribution-space Fenchel-Young loss} $L_\Omega: \RR^d\times\triangleY\to\RR_+$ is defined as:
    \begin{align*}
        L_\Omega(\thetav\,;\p)\,\triangleq\,\Omega^*(\st) + \Omega(\p) - \langle\st,\p\rangle\,=\,\maxo(\st) + \Omega(\p) - \langle\thetav,\EE_{Y\sim\p}[Y]\rangle.
    \end{align*}
\end{definition}

This loss is convex in $\thetav$, smooth if $\Omega$ is strongly convex, non-negative, and is zero when $\argmaxo(\st)=\p$.
Minimizing $L_\Omega$ uses the gradient $\nabla_\thetav L_{\Omega}(\thetav\,;\p) = \EE_{\argmaxo(\st)}[Y]\! - \!\EE_\p[Y]$, requiring a global oracle over $\cY$: a marginal oracle in the entropic case ($\argmaxo(\st)(\y)\propto\exp(\langle\thetav,\y\rangle)$), or a MAP oracle to estimate $\EE_{Z\sim\cD}[\argmax_{\y\in\cY}\langle\thetav+\varepsilon Z,\y\rangle]$ under perturbation-based regularization (\cref{sec:perturbed_lns}). Evaluating such global oracles at every step is intractable for NP-hard problems.

\paragraph{Differentiating through combinatorial optimization.}

Most approaches to differentiate Problem~(\ref{pb:ilp}) with respect to $\thetav$ introduce regularization tailored to the available oracle. If the original problem admits a dynamic programming solution, one can solve an entropy-regularized version via algorithmic smoothing \citep{li_first-_2009,mensch_differentiable_2018}. For linear programs, interior-point solvers can compute a log-barrier-regularized solution \citep{mandi_interior_2020}, while projection oracles can be used for quadratic regularization \citep{mckenzie_differentiating_2024}. Global MAP oracles can also be used to solve a regularized problem, either via Frank-Wolfe \citep{niculae_sparsemap_2018,krishnan_barrier_2015} or via Monte-Carlo estimation with input perturbations \citep{berthet_learning_2020,dalle_learning_2022}. Regarding differentiation, several strategies are possible. When the approach only needs to differentiate through a (regularized) $\max$, as for Fenchel-Young losses \citep{blondel_learning_2020}, we can use Danskin's theorem \citep{danskin_theory_1966}. When differentiating a (regularized) $\argmax$, one can implicitly differentiate KKT conditions \citep{agrawal_differentiable_2019,ferber_mipaal_2020,blondel_efficient_2022}, or compute surrogate gradients via interpolation of the solver \citep{vlastelica_differentiation_2020}.

\paragraph{Local oracles in structured prediction.} 

MAP inference in graphical models \citep{wainwright_graphical_2008} is inherently a combinatorial optimization problem. In both domains, the intractability of global inference motivates the use of local oracles. For example, the seminal iterated conditional modes algorithm \citep{besag_statistical_1986} deterministically maximizes local conditional probabilities, acting as the graphical models analogue to (unregularized) LNS. During learning, intractable global marginal inference is typically bypassed using local oracles: pseudolikelihood (or composite likelihood) methods \citep{besag_spatial_1974, lindsay11988composite, varin_overview_2011} optimize local conditionals of observed data, while local perturb-and-MAP \citep{bertasius_local_2016} forms their noise-regularized counterpart. Just like contrastive divergences \citep{hinton2002training}, which interpolate between pseudolikelihood (CD-$1$) and maximum likelihood estimation (CD-$\infty$) \citep{hyvarinen_consistency_2006,asuncion_learning_2010}, we will show that the number of RLNS iterations enables interpolating between local (RLNS-$1$) and global Fenchel-Young learning (RLNS-$\infty$). While \citet{berthet_learning_2020} successfully brought global perturb-and-MAP learning \citep{papandreou_perturb-and-map_2011} to combinatorial layers, leveraging scalable local oracles remains a critical gap in the field which our work aims to bridge.

\section{Regularized LNS}
\label{sec:regularized_lns}

\subsection{Proposed framework}
\label{sec:framework}

We now introduce our main contribution: using smoothed maximum operators, we define a stochastic, regularized version of LNS (RLNS), whose iterates follow a Markov chain with state space $\cY$. To do so, we first define a family of local regularization functions from a global one, via inclusion maps.

\begin{definition}[Local regularizers]
\label{def:local_reg}
    For any subset of solutions $\cZ\subseteq\cY$, we define the inclusion map $\iota_{\cZ}:\triangle^\cZ\lhook\joinrel\to\triangleY$ from local distributions on $\cZ$ to global distributions on $\cY$ via padding:

    \vspace{-0.7cm}
    \begin{align*}
        \iota_{\cZ}(\q)\triangleq\left(\begin{cases}
            q_\y &\text{if $\y\in\cZ$}\\
            0 &\text{else}
        \end{cases}\right)_{\y\in\cY}.
    \end{align*}

    \vspace{-0.4cm}
    Then, for any lower semicontinuous, strictly convex global regularizer $\Omega:\triangleY\to\RR$, we define its local counterpart $\Omega^{\cZ}:\triangle^\cZ\to\RR$ by composing it with $\iota_\cZ$, as $\Omega^\cZ\triangleq\Omega\circ\iota_\cZ$.
\end{definition}

While standard approaches given in \cref{tab:overview} aim to solve a regularized version of the global optimization problem \eqref{pb:ilp}, we propose to regularize the LNS subproblems from \cref{algo:lns} using the local regularizer $\Omegak$, to define local distributions $\argmaxok((\st)_{[\Yk]})\in\trianglek$. RLNS is then obtained by iteratively sampling from these locally regularized distributions, as shown in \cref{algo:rlns}.

\begin{algorithm}[H]
\caption{Regularized LNS with $K$ iterations (RLNS-$K$)}
\label{algo:rlns}
\begin{algorithmic}[1]
\STATE\textbf{Inputs:}  Logits $\thetav$, feasible set $\cY$, initial solution $\y^{(0)}$, number of iterations $K$, regularization $\Omega$.
\FOR{$k = 0$ to $K-1$}
    \STATE Choose a subset of frozen coordinates $\Sk \subset [d]$.
    \STATE Define the local neighborhood: $\Yk\triangleq \left\{\y \in \cY \mid \y_{[\Sk]}=\yk_{[\Sk]}\right\}$
    \STATE Sample $\y^{(k+1)}\sim \argmaxok\left(\stk\right)$
    \COMMENT{Sample from the local regularized distribution.}
    \ENDFOR
\STATE \textbf{Return $\frac{1}{K}\sum_{k=1}^K\yk$}
\end{algorithmic}
\end{algorithm}
\paragraph{Ergodicity.} To guarantee that the distribution of the iterates $\yk$ of \cref{algo:rlns} asymptotically converges to a \emph{unique} stationary distribution independent of the initialization $\y^{(0)}$ (i.e., that the Markov chain is \emph{ergodic}), the two following properties are sufficient:
\begin{enumerate}[topsep=0pt,itemsep=0pt,partopsep=0pt,parsep=0pt]
    \item The algorithm employs an \emph{irreducible variable selection scheme}, i.e., the distribution over frozen variables $\Sk$ induces a connected transition graph over the feasible set $\cY$.
    \item The local update distributions are \emph{dense}, i.e., $\argmaxok(\stk)(\y) > 0\; \forall \y \in \Yk$.
\end{enumerate}
The first condition ensures that LNS can reach any feasible solution from any other one. This requirement is not very stringent in practice, and is typically met by practical implementations which either consider reduced scope or exploit decomposition into subproblems with specific structure \citep{archetti2014survey}. RLNS can therefore directly leverage existing practitioner knowledge and domain-specific variable selection heuristics from standard LNS.

The second condition (satisfied, e.g., when $\Omegak$ is Legendre-type) ensures positive path probabilities and thus irreducibility. Since $\yk \in \Yk$, dense updates also yield positive self-loop probabilities, ensuring aperiodicity and therefore ergodicity \citep{levin_markov_2008}.

\begin{definition}[Stationary distribution of RLNS]
    When the Markov chain $(\yk)_{k\in\NN}$ defined by \cref{algo:rlns} is ergodic, we denote by $\piLNS\in\triangleY$ its unique stationary distribution.
\end{definition}

\paragraph{Rejection-free sampling.} In contrast to the approach of \citet{vivier-ardisson_learning_2025},
which turns local search heuristics into Metropolis-Hastings-based MCMC samplers, \cref{algo:rlns} does not use any rejection mechanism, enabling updating more variables at a time without stagnating (see \cref{fig:k_hot}).

\subsection{Entropic LNS and Gibbs sampling}
\label{sec:entropy_lns}

We now focus on the case where $\Omega$ is Shannon's negative entropy scaled by
$\gamma > 0$,
\begin{align*}
    \Omega(\q)\triangleq -\gamma H^s(\q) = \gamma\sum_{\y\in\cY}\q(\y)\log \q(\y).
\end{align*}
This choice yields the Gibbs distribution
$\argmaxo(\s)(\y)\propto\exp(s_\y/\gamma)$. Computing its expectation (i.e.,
performing global marginal inference) is only tractable in certain settings,
e.g., if the global problem admits an exact dynamic programming solution
\citep{mensch_differentiable_2018}. The local RLNS update distribution
$\argmaxok(\s_{[\Yk]})$, however, is a \emph{local} Gibbs distribution on $\Yk$, and
can be sampled from by regularized dynamic programming if the \emph{local}
problem can be solved by dynamic programming. This is the case, e.g., of our applications in \cref{sec:k-hot,sec:gap}.

\begin{proposition}[Entropic LNS as a block Gibbs sampler]
\label{prop:gibbs_mcmc}
Assume $\Omega=-\gamma H^s$ with $\gamma > 0$. For any current solution $\yk \in
\cY$ and frozen subset $\Sk \subseteq [d]$, the locally regularized distribution matches the global regularized distribution conditioned on the fixed variables:
\begin{align*}
    \argmaxok\left(\stk\right)(\y) = \argmaxo(\st)\left(\y\mid\yk_{[\Sk]}\right) \quad \forall \y \in \Yk.
\end{align*}
Consequently, \cref{algo:rlns} implements an exact \textbf{block Gibbs sampler} with $\argmaxo(\st)$ as a target distribution, and if the variable selection scheme is irreducible, we have:
\begin{align*}
    \piLNS(\y) = \argmaxo(\st)(\y)\propto\exp\left(\langle\thetav,\y\rangle/\gamma\right)\quad\forall\y\in\cY.
\end{align*}
\end{proposition}
The proof is given in \cref{proof:gibbs_mcmc}. Thus, in the Shannon-regularized case, the stationary distribution of the locally-regularized iterates of \cref{algo:rlns} matches the globally-regularized distribution. The next proposition shows that $\Omega=-\gamma H^s$ is in fact the only separable regularizer with this property. The proof, given in \cref{proof:stationary_equivalence}, relies on a functional equation-based characterization of Shannon's entropy from information theory \citep{horibe1988entropy, gselmann2011entropy,mensch_differentiable_2018}.

\begin{proposition}[Characterization of stationary exactness]
\label{prop:stationary_equivalence}
Let $\Omega:\triangleY\to\RR$ be a l.s.c. separable Legendre-type regularization function, normalized such that $\Omega(\delta^\cY_\y)=0$ for any $\y\in\cY$. The RLNS stationary distribution matches the global target for all finite feasible sets $\cY \subset \RR^d$, all maximization directions $\thetav \in \RR^d$, and all irreducible variable selection schemes, meaning
\begin{align*}
    \piLNS = \argmaxo(\st), \quad \textbf{if and only if}\quad \exists\gamma>0:\;\Omega=-\gamma H^s.
\end{align*}
\end{proposition}

\paragraph{Rao-Blackwellization.} When using regularized dynamic programming to sample the next RLNS iterate $\y^{(k+1)}$ from the local Gibbs distribution $\argmaxok(\stk)$, we can compute its expectation $\EE[\y^{(k+1)}|\yk\!,\!\Sk]$ with the same complexity. Thus, we can reduce variance by averaging local expectations instead of single samples during MCMC \citep{mckeague_markov_2000}.

\subsection{Perturbed LNS and tractable updates}
\label{sec:perturbed_lns}

In many applications, a solver for
the unregularized local problem in \cref{eq:local_map} (local MAP oracle) is
available but a solver for the Shannon-regularized local problem (local marginal oracle) is not. To support this setting,
we now consider the case of \emph{perturbation-based} regularization, which allows us to make RLNS tractable whenever standard, unregularized LNS is.

\begin{definition}[Perturbation-based regularization]
\label{def:perturbation_regularization}
    Let  $\varepsilon>0$ be a temperature parameter, and $\cD$ be a distribution on $\RR^d$ admitting a positive density. Following \citet{bouvier_primal-dual_2025}, we define the (global) distribution-space \emph{perturbation-based regularization} function $\Omegae:\triangleY\to\RR$ as:
    \begin{align*}
        \Omegae(\q)\triangleq \sup_{\s\in\RRY}\Big\{ \langle\s,\q\rangle - \EE_{Z\sim\cD}\Big[\max_{\y\in\cY} \{s_\y +\varepsilon \langle Z,\y\rangle\}\Big]\Big\}.
    \end{align*}
\end{definition}

We show in \cref{proof:strict_convexity} that $\Omegae$ is strictly convex on $\triangleY$. It is also Legendre-type when $\cY$ consists only of extreme points of $\conv(\cY)$ \citep{bouvier_primal-dual_2025}. The corresponding regularized distribution is the \emph{global} perturb-and-MAP distribution \citep{papandreou_perturb-and-map_2011,berthet_learning_2020} $\argmaxoe(\st)(\y)\!=\!\PP_{Z\sim\cD}\!\left(\argmax_{\y'\in\cY}\langle\thetav\!+\!\varepsilon Z,\y'\rangle\!=\!\y \right)$, and its mean is the expected perturbed \emph{global} maximizer $\EE_{\argmaxoe(\st)}[Y]\!=\!\EE_{Z\sim\cD}\!\left[\argmax_{\y\in\cY}\langle\thetav\!+\!\varepsilon Z,\y\rangle\right]$. We now show that the \emph{local} regularized distributions correspond to \emph{local} perturb-and-MAP distributions.

\begin{proposition}[Perturbed LNS sampling]
\label{prop:perturbation_local}
    Let $\Omegae$ be a perturbation-based regularizer, as defined in \cref{def:perturbation_regularization}. Exact sampling from the locally regularized distribution on $\Yk$ with local scores $\stk \in \RR^{|\Yk|}$ is tractable via perturbed \emph{local} maximization:
    \begin{align*}
        Z \sim \cD \implies \y^{(k+1)}\triangleq\argmax_{\y\in\Yk}\langle\thetav+\varepsilon Z,\y\rangle \sim \argmaxoek\left(\stk\right).
    \end{align*}
\end{proposition}

\paragraph{Tractable updates.}
\cref{prop:perturbation_local}, proved in \cref{proof:perturbation_local}, highlights the computational advantage of using $\Omegae$: samples from the local regularized distributions can be obtained by performing perturbed \emph{local} maximization, leveraging the standard LNS oracles from \cref{eq:local_map} and \cref{algo:lns}. We summarize this in \cref{algo:perturbed_lns}, whose iterates match those of \cref{algo:rlns} for $\Omega=\Omegae$.

\begin{algorithm}[H]
\caption{Perturbed LNS}
\label{algo:perturbed_lns}
\begin{algorithmic}[1]
\STATE\textbf{Inputs:}  Logits $\thetav \in \RR^d$, finite feasible set $\cY \subset \RR^d$, noise distribution $\cD$, temperature parameter $\varepsilon>0$, initial solution $\y^{(0)} \in \cY$, number of iterations $K\in\NN$.
\FOR{$k = 0$ to $K-1$}
    \STATE Choose a subset of coordinates $\Sk \subset [d]$.
    \STATE Define the local neighborhood: $\Yk\triangleq \left\{\y \in \cY \mid \y_{[\Sk]}=\yk_{[\Sk]}\right\}$
    \STATE Sample $Z^{(k)}\sim \cD$
    \STATE Set $\y^{(k+1)}\gets \argmax_{\y\in\Yk}\langle\thetav+\varepsilon Z^{(k)},\y\rangle$
\ENDFOR
\STATE \textbf{Return $\frac{1}{K}\sum_{k=1}^K\yk$}
\end{algorithmic}
\end{algorithm}

\section{Learning with Regularized Large Neighborhood Search}
\label{sec:lns_fyl}

In this section, we show how to train models using RLNS as an output combinatorial layer. 

\subsection{RLNS-$\infty$ and global Fenchel-Young loss}
\label{sec:rlns_inf}

In this section, we consider the asymptotic regime where the number of RLNS iterations grows large ($K\to\infty$). Let $\Omega:\triangleY\to\RR$ be a global distribution-space regularizer, and $L_\Omega: \RR^d \times \triangleY \to \RR_+$ be the corresponding global Fenchel-Young loss, as defined in \cref{def:dist_fyl}. For any target distribution\footnote{In practice, $\p$ is typically a Dirac distribution $\delta_\y^\cY$ on some target $\y\in\cY$, such that $\EE_\p[Y]=\y$.} $\p \in \triangleY$, the gradient of this loss is given by $\nabla_\thetav L_\Omega(\thetav\,;\p) = \EE_{\argmaxo(\st)}[Y] - \EE_\p[Y]$. Since the expectation of the global regularized distribution $\argmaxo(\st)$ is intractable, we can use the RLNS iterates $(\yk)_{k\in\NN}$ of \cref{algo:rlns} to construct an MCMC estimator of this gradient:
\begin{align*}
    \nabla_\thetav L_\Omega(\thetav\,;\p) \approx \frac{1}{K}\sum_{k=1}^K\yk - \EE_\p[Y].
\end{align*}

The behavior and bias of this stochastic gradient estimator depend on the chosen regularization, variable selection scheme, and number of iterations $K$.

\paragraph{Shannon-regularized case.} When $\Omega = -\gamma H^s$, \cref{prop:gibbs_mcmc} guarantees that the stationary distribution of the RLNS Markov chain exactly matches the target distribution, i.e., $\piLNS = \argmaxo(\st)$. In this setting, the only source of bias in the gradient estimator stems from the finite mixing time of the Markov chain. While biased for any fixed $K$, the estimator is \emph{consistent} in the $K\to\infty$ limit.

\paragraph{General case.} For a general $\Omega$, RLNS accumulates two sources of bias: one due to the mixing time of the chain, and one due to the intrinsic discrepancy between $\piLNS$ and $\argmaxo(\st)$. In \cref{sec:approximation_guarantees}, we characterize the latter through the \emph{local consistency gap}: the worst-case mismatch between a single RLNS update and the one of an ideal block Gibbs sampler targeting $\argmaxo(\st)$ (\cref{def:consistency_gap}). The bias scales with this gap, which in turn decreases as $\Omega$ approaches $-\gamma H^s$ (\cref{prop:bias_bound,prop:consistency_bound_smoothness}).

In the case of perturbed LNS, analyzing the problem's \emph{constraint hypergraph} (\cref{def:constraint_graph}) shows that this gap vanishes when the feasible set decomposes into independent variable blocks aligned with the neighborhood selection, and otherwise grows with the number of constraints coupling frozen and free variables (\cref{prop:product_space_exact,prop:cut_bound}). This motivates ``min-cut'' variable selection schemes that minimize shared constraints between frozen and free variables (\cref{sec:mwis,fig:mwis}).

\paragraph{The choice of $K$.} The number of iterations $K$ allows us to interpolate between two learning paradigms. Setting $K=1$ minimizes computational overhead and provides unbiased gradients for a local Fenchel-Young loss, a generalized form of pseudolikelihood (\cref{sec:rlns_1}). Conversely, a larger $K$ better estimates global Fenchel-Young gradients but incurs higher computational costs.

\subsection{RLNS-$1$ and local Fenchel-Young loss}
\label{sec:rlns_1}

We study the special case of a single RLNS iteration. For any target solution $\y\in\cY$ and fixed coordinates $S\subseteq[d]$, we denote the corresponding neighborhood by $\cY_S\triangleq\{\y'\in\cY\mid \y'_{[S]}=\y_{[S]}\}$.

\begin{definition}[Local Fenchel-Young loss]
\label{def:local_fyl}
     For any distribution $\nu\!\in\!\triangle^{\!\cP([d])}$ on coordinate subsets, define the expected smoothed local maximum value function $F_\y:\RR^{|\cY|}\to\RR\cup\{+\infty\}$ as:
    \begin{align*}
    F_\y(\s)\triangleq\EE_{S\sim\nu}\left[f_{\Omega^{\cY_S}}\left(\s_{\left[\cY_S\right]}\right)\right].
    \end{align*}
    Letting its convex conjugate be $\Omega_\y\triangleq(F_\y)^*$, the corresponding target-dependent, \emph{local Fenchel-Young loss} $L_{\Omega_\y}:\RR^d\times\dom(\Omega_\y)\to\RR_+$ is given by:
    \begin{align*}
        L_{\Omega_\y}(\thetav\,;\p)\triangleq F_\y(\st) + \Omega_\y(\p) - \langle\thetav,\EE_\p[Y]\rangle.
    \end{align*}
\end{definition}

The following proposition, which we prove in \cref{proof:one_step_grad}, shows we can obtain unbiased stochastic gradients for this loss directly from a single step of RLNS.

\begin{proposition}[Unbiased gradients via RLNS-$1$]
\label{prop:one_step_grad}
    For any $\y\in\cY$, we have $\delta^\cY_\y \in \dom(\Omega_\y)$. Moreover, the gradient of the local Fenchel-Young loss with target $\y$ is given by:
    \begin{align*}
        \nabla_\thetav L_{\Omega_\y}(\thetav\,;\delta^\cY_\y) = \EE_{S\sim\nu}\Big[ \EE_{\p_{\Omega^{\cY_S}} \big((\st)_{\left[\cY_S\right]}\big)}[Y]\Big] - \y.
    \end{align*}
    Consequently, executing \cref{algo:rlns} with ground-truth initialization $\y^{(0)}=\y$ and a single iteration $K=1$ yields an unbiased estimator of $\nabla_\thetav L_{\Omega_\y}(\thetav\,;\delta^\cY_\y)$. 
\end{proposition}

\paragraph{Connection with pseudolikelihood.} 
In the Shannon-regularized setting $\Omega=-\gamma H^s$, this local Fenchel-Young loss recovers the pseudolikelihood (a.k.a.\ composite likelihood) objective \citep{besag_spatial_1974},  which maximizes the likelihood of subsets of observed variables conditioned on the rest.

\begin{proposition}[Equivalence to pseudolikelihood]
\label{prop:composite_likelihood}
    When $\Omega=-\gamma H^s$, the local Fenchel-Young loss with target $\y$ equals the expected negative log-likelihood of $\y$ conditioned on $\y_{[S]}$:
    $$ L_{\Omega_\y}(\thetav\,;\delta^\cY_\y) = -\gamma\cdot\EE_{S\sim\nu}\left[ \log \left(\argmaxo(\st)(\y\mid\y_{[S]})\right)\right]. $$
\end{proposition}

The derivation is given in \cref{proof:composite_likelihood}. Thus, since the global Fenchel-Young loss is given by $L_\Omega(\thetav\,;\delta^\cY_\y)=-\gamma\log\argmaxo(\st)(\y)$ in this setting, the number of entropic RLNS iterations enables interpolating between pseudolikelihood and exact maximum likelihood estimation.

\section{Experiments}
\label{sec:experiments}

\subsection{Estimation on $\kappa$-hot vectors}
\label{sec:k-hot}

To validate our proposed estimators in a controlled setting, we consider the set of $\kappa$-hot vectors, $\cY\triangleq \{\y\in\{0,1\}^d\mid \sum_{i=1}^dy_i=\kappa\}$. In this context, the expectation of the global regularized distribution $\EE_{\argmaxo(\st)}[Y]$ can be exactly computed via regularized dynamic programming for entropic regularization \citep{vivier-ardisson_differentiable_2026}, or accurately estimated via exact i.i.d.\ Monte-Carlo sampling for perturbation-based regularization. This provides a reliable ground truth against which we can evaluate our approach. Full experimental details are provided in \cref{sec:appendix_k_hot_details}.

\paragraph{Baselines.} We compare RLNS against exact i.i.d.\ Monte-Carlo estimation (sampling directly from the global perturb-and-MAP \citep{papandreou_perturb-and-map_2011,berthet_learning_2020} or Gibbs distribution) and Local Search MCMC \citep{vivier-ardisson_learning_2025} using the same neighborhoods. Drawing exact global Monte-Carlo samples requires a global oracle, which is tractable here on $\kappa$-hot vectors by design, but intractable on general NP-hard problems. Thus, i.i.d.\ Monte-Carlo serves as an idealized upper bound on sample efficiency.

\paragraph{Results.} Tracking the mean absolute error (MAE) to $\EE_{\argmaxo(\st)}[Y]$ (\cref{fig:k_hot}), i.i.d.\ Monte-Carlo naturally converges faster per sample than RLNS, whose Markov chain iterates are correlated and update only free variables via local oracles. Leaving more variables free accelerates RLNS convergence toward this global ceiling at the cost of larger subproblems (\cref{fig:k_hot}, left). Moreover, RLNS converges much faster than Local Search MCMC (\cref{fig:k_hot}, center), as its $100\%$ acceptance rate avoids Metropolis-Hastings stagnation on larger neighborhoods or low $\gamma$ (\cref{fig:k_hot}, right), and Rao-Blackwellization further reduces early-iteration error (\cref{fig:k_hot}, center).

\begin{figure}
    \centering
    \includegraphics[width=0.99\linewidth]{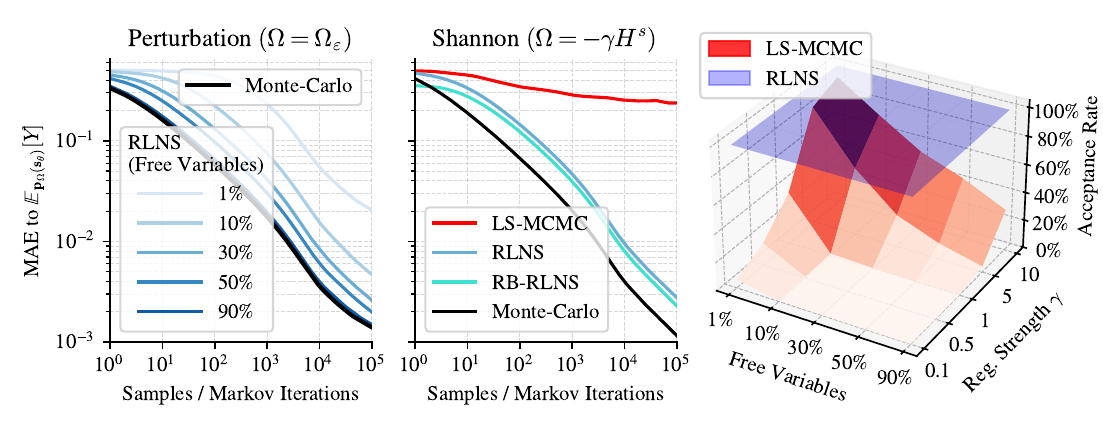}
    \caption{ Estimation on $\kappa$-hot vectors. \textbf{Left:} convergence of RLNS with varying proportion of free variables compared to i.i.d.\ Monte-Carlo estimation. \textbf{Center:} convergence with $30\%$ of free variables, against Local Search MCMC and i.i.d.\ Monte-Carlo estimation (RB- denotes Rao-Blackwellization). \textbf{Right:} acceptance rate with varying proportion of free variables and regularization strength $\gamma$.}
    \label{fig:k_hot}
\end{figure}

\subsection{Generalized assignment problem}
\label{sec:gap}

We now evaluate RLNS as a combinatorial layer on the generalized assignment problem (GAP) \citep{martello_generalized_1992}. In the GAP, we are given $n$ items and $m$ bins with capacity $C_j\in\NN$. Each item-bin assignment has value $\theta_{ij}\in\RR$ and a weight $w_{ij}\in\NN$. The combinatorial problem writes:
\begin{align*}
    \max_{\y\in\{0,1\}^{n\times m}}\;&\sum_{i=1}^n\sum_{j=1}^m\theta_{ij}y_{ij} &&\text{s.t.} \quad \sum_{j=1}^m y_{ij}  \leq 1 \quad\forall i\in[n],\quad \text{and} \quad\sum_{i=1}^n w_{ij}y_{ij}  \leq C_j \quad\forall j\in[m].
\end{align*}
The first constraint ensures each item is assigned to at most one bin, while the second one ensures that the capacity of each bin is not exceeded. Given a dataset of $(\x,\y)$ pairs where $\x\in\RR^{n\times m\times p}$ contains relational features for each item-bin pair, the goal is to learn a neural network $f_W:\RR^p\to\RR$ that predicts assignment scores $\theta_{ij} = f_W(\x_{ij})$, such that $\hat{\y}(\thetav)\approx \y$.

\paragraph{Local oracles.} We choose to freeze the assignments to all bins except a randomly selected one at each RLNS iteration. The LNS subproblem is then a $0/1$ one-dimensional Knapsack problem, which admits an exact dynamic programming solution \citep{kellerer_knapsack_2004}, and we can sample from the local regularized distributions via perturbed or regularized dynamic programming \citep{vivier-ardisson_differentiable_2026}. Detailed explanations and full experimental details are given in \cref{sec:appendix_gap_details}.

\paragraph{Results.} In \cref{fig:gap_iterations}, we intentionally compare perturbed RLNS directly against global perturbed Fenchel-Young learning (PFY, \citet{berthet_learning_2020}) to isolate how the number of local updates bridges local and global learning: increasing $K$ steadily lowers test MAE and closes the gap to global PFY despite the stationary discrepancy of perturbed LNS. Increasing $K$ also reduces stochastic gradient variance, which is further reduced under entropic regularization via Rao-Blackwellization (\cref{fig:gap_time_variance}, top). Finally, comparing the wall-clock time of a single oracle call across problem scales ($m=n/10$), local DP-based MAP oracles are orders of magnitude faster than global MILP solvers, with the speedup growing sharply with $n$ (\cref{fig:gap_time_variance}, bottom).

\begin{figure}[t!]
    \centering
    \begin{subfigure}[b]{0.325\textwidth}
        \centering
        \includegraphics[width=\textwidth]{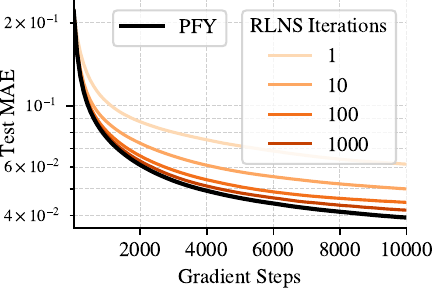}
        \caption{Average test set MAE between predicted and ground-truth solutions of RLNS and PFY on the GAP.}
        \label{fig:gap_iterations}
    \end{subfigure}
    \hfill
    \begin{minipage}[b]{0.325\textwidth}
        \begin{subfigure}[b]{\textwidth}
            \centering
            \includegraphics[width=\textwidth]{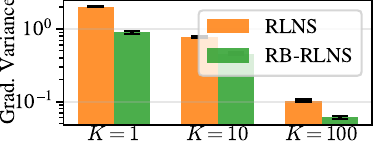}
        \end{subfigure}
        
        \vspace{-0.05cm} %
        
        \begin{subfigure}[b]{\textwidth}
            \centering
            \includegraphics[width=\textwidth]{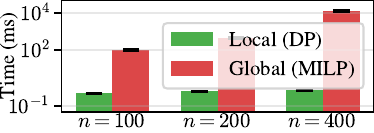}
            \caption{Variance of RLNS gradients and oracle wall-clock time on the GAP.}
            \label{fig:gap_time_variance}
        \end{subfigure}
    \end{minipage}
    \hfill
    \begin{subfigure}[b]{0.325\textwidth}
        \centering
        \includegraphics[width=\textwidth]{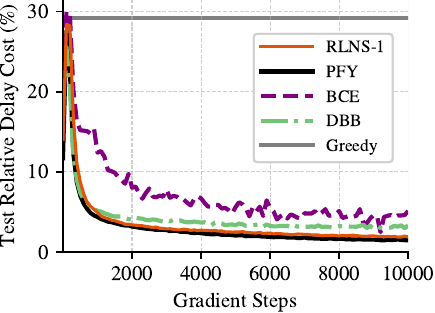}
        \caption{Delay cost relative to the expert policy on stochastic vehicle scheduling test set instances.}
        \label{fig:svsp}
    \end{subfigure}
    \caption{Empirical evaluation of RLNS as a combinatorial layer on the generalized assignment problem (a, b) and stochastic vehicle scheduling (c).}
    \label{fig:main}
\end{figure}

\subsection{Stochastic vehicle scheduling}
\label{sec:svsp}

We now evaluate RLNS as a combinatorial layer on a stochastic vehicle scheduling benchmark \citep{parmentier_learning_2022,dalle_learning_2022,hoppe_structured_2025}. The objective is to route a fleet of vehicles in order to serve a set of spatio-temporal tasks, minimizing the sum of the deterministic vehicle cost and the expected delay cost across a set of stochastic scenarios. The problem is framed as finding paths on a directed acyclic graph $\cG=(\cV, \cE)$, where a feasible routing $\y \in \{0,1\}^{|\cE|}$ satisfies flow conservation and unit demand constraints. To imitate the ground-truth expert (the optimal solution of a two-stage stochastic MILP with full scenario knowledge), a feed-forward neural network predicts edge costs $\theta_e$ from edge-level features, and the predicted routing is obtained via a flow LP. Complete experimental details are provided in \cref{sec:appendix_svsp_details}.

\paragraph{Approach and baselines.} 
We use RLNS-$1$ with perturbation-based regularization. At each step, we freeze the edges corresponding to $66\%$ of randomly selected vehicle routes, and exactly re-optimize the remaining free variables using a significantly smaller flow LP. We compare RLNS against global perturbed Fenchel-Young learning (PFY, \citet{berthet_learning_2020}), blackbox differentiation (DBB, \citet{vlastelica_differentiation_2020}), an unstructured binary cross-entropy loss  (BCE, see, e.g., \citet{joshi_efficient_2019}), and a non-learned greedy policy that minimizes only the deterministic vehicle cost.

\paragraph{Results.} 
All learned methods quickly reach the optimal vehicle cost by using the same number of vehicles as the expert. However, in terms of delay cost (\cref{fig:svsp}), RLNS-$1$ outperforms DBB and BCE and nearly matches global PFY while only solving flow subproblems of roughly one-third the original size. On these $n=50$ instances (the largest scale tractable for the ground-truth two-stage stochastic MILP used to generate data), a single local MAP call averages $1.8\text{ ms}$ versus $4.4\text{ ms}$ for the global MAP oracle, a gap that widens as problem size grows.

\section*{Conclusion}

We introduced regularized large neighborhood search (RLNS), turning standard LNS heuristics into rejection-free MCMC samplers to provide gradients through combinatorial optimization layers using only scalable local oracles, and interpolating between local composite likelihood ($K=1$) and global Fenchel-Young learning ($K\to\infty$).

\paragraph{Bridging communities and future work.}

While we exhibit links between LNS and Gibbs sampling and prior work bridged local search algorithms with Metropolis-Hastings \citep{vivier-ardisson_learning_2025}, we believe more combinatorial optimization metaheuristics can be revisited from a sampling perspective. For instance, LNS can use local search heuristics as approximate subproblem solvers in practice \citep{pisinger2018large}, suggesting a natural bridge between ``local-search-within-LNS'' metaheuristics and ``Metropolis-within-Gibbs'' samplers \citep{roberts_geometric_1997,ascolani_scalability_2026}. Differently, the systematic design of optimal variable selection schemes tailored to a given problem structure and regularization choice remains an important open direction.

\paragraph{Limitations.} First, RLNS faces an oracle / exactness trade-off: entropic RLNS guarantees exact block Gibbs sampling equivalence but requires local marginal oracles, whereas perturbed RLNS relies only on local MAP oracles at the cost of a stationary bias when constraints couple frozen and free variables (\cref{sec:approximation_guarantees}), for which our theoretical bounds are structural rather than numerically tight. Second, our learning formulation focuses on Fenchel-Young losses, and the unbiased $K=1$ regime requires target solutions $\y\in\cY$ for initialization, which may not be directly available in contexts such as reinforcement or decision-focused learning without supervision.

\section*{Acknowledgments and disclosures}
This work was fully supported by Google DeepMind, where Germain Vivier-Ardisson is a full-time employee and PhD candidate. The authors have no competing interests or financial relationships outside the submitted work to disclose.

\bibliography{main}
\bibliographystyle{abbrvnat}

\newpage
\appendix

\section*{Notation}

\begin{itemize}
    \item We denote $\RR_+\triangleq[0,+\infty)$, $[D]\triangleq\{1,\dots,D\}$ for any integer $D\in\NN$, and $\cP(S)$ for the power set of any set $S$. We denote by $\0$, $\1$, and $\I$ the all-zeros vector, all-ones vector, and identity matrix (of appropriate dimensions).
    \item For any finite index set $\cZ$ (such as $[D]$), vector $\x=(x_z)_{z\in\cZ}\in\RR^{|\cZ|}$, and subset $I\subseteq\cZ$, we denote $\x_{[I]}\triangleq(x_i)_{i\in I}\in\RR^{|I|}$. Similarly, for any matrix $\A\in\RR^{M\times D}$ and subsets $I\subseteq[M]$, $J\subseteq[D]$, we denote $\A_{[I,J]}\triangleq(A_{i,j})_{i\in I,j\in J}\in\RR^{|I|\times|J|}$ (and $\A_{[\colon,j]}\triangleq\A_{[[M],\{j\}]}$).
    \item For any finite set $\cY$, we denote by $\triangleY\triangleq\{\p\in\RR_+^{|\cY|}\mid\sum_{\y\in\cY}p_\y=1\}$ the set of categorical distributions on $\cY$. For $\p\in\triangleY$, we equivalently write $\p(\y)$ or $p_\y$ for the probability of $\y\in\cY$, $\supp(\p)\triangleq\{\y\in\cY\mid p_\y>0\}$ for its support, and $\delta^\cY_\y\in\triangleY$ for the Dirac distribution at $\y\in\cY$.
    \item For any subset of coordinates $I\subseteq[d]$, any solution $\y'\in\cY$ and $\y\in\cY$ such that $\y_{[I]}=\y'_{[I]}$\,, we denote by:
    \begin{align*}
        \p(\y\mid\y'_{[I]})\triangleq \frac{\p(\y)}{\sum_{\substack{\y''\in\cY\\\text{s.t. } \y''_{[I]}=\y'_{[I]}}}\p(\y'')}
    \end{align*}
    the corresponding conditional probability (we always use the functional notation for conditional probabilities).
    \item For any $\p, \q \in \triangleY$, we denote their total variation (TV) distance, and for $\q\in\ri(\triangleY)$ their Kullback-Leibler (KL) divergence (with the convention $0\log 0 = 0$), by:
    \begin{align*}
        \TV(\p,\q) \triangleq \frac{1}{2} \sum_{\y \in \cY} | p_\y - q_\y |,\qquad
        \KL(\p\|\q) \triangleq \sum_{\y \in \cY} p_\y \log\left(\frac{p_\y}{q_\y}\right).
    \end{align*}
    \item We denote the complement of any event $E$ by $\overline{E}$, and its indicator by $\mathds{1}_E\in\{0,1\}$.
    \item For any function $f:\RR^D\to\RR\cup\{+\infty\}$, we denote by $\dom(f)\triangleq\{\x\in\RR^D\mid f(\x)<+\infty\}$ its domain, and its convex conjugate by
    $$f^*(\y)\triangleq \sup_{\x\in\dom(f)}\{\langle\x,\y\rangle-f(\x)\}.$$
    \item For any convex function $f:\cC\to\RR$ on a convex set $\cC\subseteq\RR^D$ that is differentiable on $\ri(\cC)$, we denote the Bregman divergence generated by $f$ as
    $$B_f(\x\|\y)\triangleq f(\x)-f(\y)-\langle\nabla f(\y),\x-\y\rangle\quad\forall (\x,\y)\in\cC\times\ri(\cC).$$
    \item For any function $F:\RR^D\to\RR^N$, we denote by $\range(F)\triangleq\{F(\x)\mid\x\in\RR^D\}$ its range.
    \item We denote the convex hull of any set $\cY\subseteq\RR^D$ by $\conv(\cY)$, and the relative interior of a convex set $\cC\subseteq\RR^D$ by $\ri(\cC)$.
\end{itemize}

\newpage
\section{Geometrical insights and approximation bounds for general regularizers}
\label{sec:approximation_guarantees}

When $\Omega\neq-\gamma H^s$, the locally regularized distributions used by RLNS may differ from the conditional distributions of the global smoothed maximizer $\argmaxo(\st)$, thereby precluding the block Gibbs sampling perspective established in \cref{prop:gibbs_mcmc}. In this section, we give geometrical insights into this more general case, and establish bounds on the distance between the algorithm's stationary distribution $\piLNS$ and the target $\argmaxo(\st)$.

\subsection{Geometrical insights}

First, we show that the update distributions of RLNS and the exact Gibbs sampler with target $\argmaxo(\st)\in\triangleY$ are in fact two different projections of the latter onto the local simplex $\trianglek$. 

\begin{proposition}[RLNS and Gibbs updates as Bregman projections]
    \label{prop:updates_as_projections}
    Let $\Omega:\triangleY\to\RR$ be a l.s.c. Legendre-type regularizer, $\Yk\subseteq\cY$ be the neighborhood used by \cref{algo:rlns} at step $k$, and $B_\Omega:\triangleY\times\ri(\triangleY)\to\RR_+$ denote the Bregman divergence generated by $\Omega$.
    The local RLNS update distribution corresponds to the projection of the global smoothed maximizer $\argmaxo(\st)$ onto the local simplex $\trianglek$ with respect to $B_\Omega$:
    \begin{align*}
        \argmaxok\left(\stk\right) = \argmin_{\q \in \trianglek} B_\Omega\left(\iotak(\q)\| \argmaxo(\st)\right).
    \end{align*}
    Similarly, the update distribution of the ideal block Gibbs sampler with target $\argmaxo(\st)$ is its projection onto the local simplex with respect to the Kullback-Leibler divergence:
    \begin{align*}
        \argmaxo(\st)\left(\cdot \mid \yk_{[\Sk]}\right) = \argmin_{\q \in \trianglek} \KL\left(\iotak(\q)\| \argmaxo(\st)\right).
    \end{align*}
\end{proposition}

\cref{prop:updates_as_projections}, proved in \cref{proof:updates_as_projections}, shows that RLNS and the ideal Gibbs sampler update distributions both correspond to Bregman projections of the same target distribution $\argmaxo(\st)$ onto the local simplex $\trianglek$, simply with distinct geometries. While RLNS uses the same regularization $\Omega$ as the one used by the global target distribution, Gibbs sampling uses Shannon's negative entropy.

This gives a geometrical point of view on the claim of \cref{prop:gibbs_mcmc}. Indeed, since we have $\KL=\frac{1}{\gamma}B_{-\gamma H^s}$ where $B_{-\gamma H^s}$ is the Bregman divergence generated by Shannon's scaled negative entropy, the corresponding projection problems are equivalent when $\Omega=-\gamma H^s$ in the first place: the $B_\Omega$-projections match $\KL$-projections, and RLNS matches an exact Gibbs sampler for $\argmaxo(\st)$.

We now give a block-coordinate ascent (BCA) perspective on the expected RLNS updates.

\begin{proposition}[Expected RLNS as block-coordinate ascent]
\label{prop:mean_space_bca}
Assume that $\cY \subseteq \{0,1\}^d$, and let $\cM \triangleq \conv(\cY)$. Let $\Omega : \triangleY \to \RR$ be a l.s.c. strictly convex distribution-space regularizer. Its mean-space counterpart $\Omega_\cM : \cM \to \RR$, defined as:
\begin{align*}
    \Omega_\cM(\muv) \triangleq \min_{\p \in \triangleY \text{ s.t. } \EE_\p[Y] = \muv} \Omega(\p),
\end{align*}
is strictly convex on $\cM$. Furthermore, for any current solution $\yk \in \cY$ and subset of frozen coordinates $\Sk \subseteq [d]$,  the expected RLNS local update matches the exact BCA update of the regularized mean-space objective over the free variables:
\begin{align*}
    \EE_{\argmaxok\left(\stk\right)}[Y]=\argmax_{\muv \in \cM \text{ s.t. } \muv_{[\Sk]} = \yk_{[\Sk]}} \left\{ \langle \thetav, \muv \rangle - \Omega_\cM(\muv) \right\}.
\end{align*}
\end{proposition}

The proof is given in \cref{proof:mean_space_bca}. 

\subsection{Local consistency gap}

We first define the \emph{local consistency gap}, which measures the maximum discrepancy between an RLNS update and the ideal Gibbs sampling update that would target the global regularized distribution.

\begin{definition}[Local consistency gap]
\label{def:consistency_gap}
    Let $\yk \in \cY$, $\Sk \subseteq [d]$, and let $\Yk$ denote the corresponding subset of feasible solutions, as defined in \cref{algo:lns,algo:rlns}.
    
    We define the \emph{local consistency gap} $\gap$ as the maximum TV distance between the global regularized distribution $\argmaxo(\st)$ conditioned on the subset of structures matching $\yk$ on coordinates $\Sk$ and the local RLNS update distribution (seen as two elements of $\trianglek$):
    \begin{align*}
        \gap \triangleq \max_{\yk, \Sk} \TV\left( \argmaxo(\st)(\cdot \mid \yk_{[\Sk]})\,,\argmaxok(\stk) \right).
    \end{align*}
\end{definition}

As shown in \cref{sec:entropy_lns}, if $\Omega=-\gamma H^s$ for some $\gamma>0$, the local updates match the conditionals exactly, implying $\gap = 0$. For general regularizers, $\gap \geq 0$. We now bound the total bias of the algorithm.

\begin{proposition}[Bias of RLNS]
\label{prop:bias_bound}
    Let $P$ be the transition kernel of the RLNS Markov chain defined in \cref{algo:rlns}. Assuming $\supp(\argmaxo(\st))=\cY$, let $P^*$ be the transition kernel of the ideal Gibbs sampler targeting $\argmaxo(\st)$, i.e., which uses updates of the form:
    \begin{align*}
        \y^{(k+1)}\sim \argmaxo(\st)\left(\cdot\mid \y^{(k)}_{[\Sk]} \right).
    \end{align*}
    
    Let $\tau_m(P^*)$ be the Dobrushin contraction coefficient of $(P^*)^m$, i.e.:
    \begin{align*}
        \tau_m(P^*) \triangleq \max_{\y, \y' \in \cY} \TV\left( (P^*)^m(\y\,, \cdot)\,, (P^*)^m(\y', \cdot) \right).
    \end{align*}
    Let $m \geq 1$ such that $\tau_m(P^*) < 1$. Then, the distance between the RLNS stationary distribution $\piLNS$ and the target is bounded by:
    \begin{align*}
        \TV\left( \piLNS\,,\argmaxo(\st) \right) \leq \frac{m}{1 - \tau_m(P^*)} \gap.
    \end{align*}
\end{proposition}

\cref{prop:bias_bound}, which is proved in \cref{proof:bias_bound}, enables us to now focus on bounding the local consistency gap $\gap$ as it translates to a global bound on the distance between $\piLNS$ and $\argmaxo(\st)$.

We first bound the local consistency gap under the assumption that the deviation of the regularization from the scaled Shannon's negative entropy $\Omega-(-\gamma H^s)$ has a Lipschitz gradient with respect to the $\ell_1$ norm.

\begin{proposition}[Consistency gap bound via relative smoothness]
\label{prop:consistency_bound_smoothness}
    Let $\gamma > 0$ be a reference temperature, and $h_\gamma: \triangleY \to \RR$ be the deviation of the regularization $\Omega$ from the scaled Shannon's negative entropy:
    \begin{align*}
        h_\gamma(\p) \triangleq \Omega(\p) - (-\gamma H^s)(\p) = \Omega(\p) + \gamma \sum_{\y \in \cY} p_\y \log p_\y.
    \end{align*}
    If $h_\gamma$ is $L_{\Omega,\gamma}$-smooth with respect to the $\ell_1$ norm, we have:
    \begin{align*}
        \gap \leq \frac{L_{\Omega,\gamma}}{\gamma}.
    \end{align*}
\end{proposition}

The proof is given in \cref{proof:consistency_bound_smoothness}, and leverages the characterization of the RLNS and Gibbs updates as Bregman projections given in \cref{prop:updates_as_projections}.

We now bound the consistency gap $\gap$ in the special case of perturbation-based regularizers.

\begin{proposition}[Consistency gap bound in the perturbation-based setting]
\label{prop:consistency_bound_perturbation}
    Let $\Omegae$ be a perturbation-based regularization function, as defined in \cref{def:perturbation_regularization}. Let $\argmaxoe(\st)$ denote the corresponding global regularized distribution.
    
    The local consistency gap $\gape$ is bounded by the maximum probability mass of the global regularized distribution $\argmaxoe(\st)$ falling outside a local neighborhood:
    \begin{align*}
        \gape \leq \max_{\yk, \Sk} \argmaxoe(\st)(\cY\setminus\Yk),
    \end{align*}
    where the maximum is taken over all feasible solutions $\yk \in \cY$ and subsets of coordinates $\Sk \subseteq [d]$, which define the neighborhood $\Yk$.
\end{proposition}

The proof is given in \cref{proof:consistency_bound_perturbation}.

\subsection{Constraint graph and cut-bounded bias}
\label{sec:constraint_graph}

In this section, we analyze the bias of perturbed LNS (i.e., RLNS with perturbation-based regularization as defined in \cref{sec:perturbed_lns}) through the lens of its variable selection scheme. We begin with an idealized scenario: when the feasible set perfectly factorizes into independent blocks of variables, perturbed LNS achieves zero bias if its variable selection scheme is aligned to this decomposition.

\begin{proposition}[Exactness of perturbed LNS on product spaces]
\label{prop:product_space_exact}
    Let the coordinates $[d]$ be partitioned into $M$ disjoint blocks $[d] = \bigsqcup_{m=1}^M B_m$, and assume the feasible set factorizes accordingly: $\cY = \prod_{m=1}^M \cY_m$, where $\cY_m \subset \RR^{|B_m|}$. 
    
    Assume that the noise distribution $\cD$ factorizes over this partition, meaning that if $Z \sim \cD$, the block subvectors $(Z_{[B_m]})_{m=1}^M$ are mutually independent.
    
    If, at each step $k$, the frozen subset of coordinates $\Sk$ is a union of blocks, i.e., if we have $\Sk = \bigsqcup_{m \in \Mk} B_m$ for some index set $\Mk \subset [M]$, then $\gape = 0$, and \cref{algo:perturbed_lns} implements an exact block Gibbs sampler, giving $\piLNSe=\argmaxoe(\st)$.
\end{proposition}

We give the proof in \cref{proof:product_space_exact}. While perfect factorization guarantees zero bias, most combinatorial problems feature interacting variables. Typically, the feasible set is defined by $m$ linear constraints $C\triangleq\{1,\dots,m\}$:
\begin{align*}
    \cY \triangleq \{\y \in \{0,1\}^d \mid \A\y \leq \b\},
\end{align*}
where $\A\in\RR^{m\times d}$ and $\b\in\RR^m$. We analyze these interactions via the problem's constraint graph.

\begin{definition}[Constraint hypergraph]
\label{def:constraint_graph}
    The \emph{constraint hypergraph} associated with the feasible set $\cY$ is a hypergraph $\cH_C = (\cV, \cE_C)$, where the vertices correspond to the optimization variables $\cV = [d]$, and each hyperedge $e_c \in \cE_C$ corresponds to a constraint $c \in C$, containing exactly the variables involved in that constraint:$$ \cE_C \triangleq \{e_c \mid c \in C\} \quad \text{where} \quad e_c \triangleq \{i \in [d] \mid A_{c, i} \neq 0\}. $$
\end{definition}

At any iteration $k$, the choice of frozen variables $\Sk \subseteq [d]$ and free variables $\Fk \triangleq [d] \setminus \Sk$ essentially partitions the vertices of this constraint hypergraph. This partition naturally divides the constraint indices $C$ (and thus the hyperedges) into three disjoint sets:
\begin{itemize}
    \item $C_F(\Sk)$: the constraints exclusively involving free variables in $\Fk$;
    \item $C_S(\Sk)$: the constraints exclusively involving frozen variables in $\Sk$;
    \item $C_\partial(\Sk)$: the \emph{cut constraints}, which involve at least one variable in $\Fk$ and one in $\Sk$ (corresponding to the hyperedges in the cut-set of $\cH_C$).
\end{itemize}

We can therefore decompose the constraints and write the cut constraints evaluated on a vector $\y \in \RR^d$ as $\A_{[C_\partial(\Sk), \Fk]} \y_{[\Fk]} + \A_{[C_\partial(\Sk), \Sk]} \y_{[\Sk]} \leq \b_{[C_\partial(\Sk)]}$.

We now define a relaxed space where these cut constraints are dropped. This effectively forces the subproblem back into the idealized product-space setting described in \cref{prop:product_space_exact}.

\begin{definition}[Relaxed product space]
\label{def:relaxed_product_space}
    For a given partition defined by $\Sk$, we define the \emph{relaxed feasible set} $\cY_{\text{relax}}(\Sk) \triangleq \cY_F(\Sk) \times \cY_S(\Sk)$, where:
    \begin{align*}
        \cY_F(\Sk) &\triangleq \left\{\y_F \in \{0,1\}^{|\Fk|} \mid \A_{[C_F(\Sk), \Fk]} \y_F \leq \b_{[C_F(\Sk)]}\right\}, \\
        \cY_S(\Sk) &\triangleq \left\{\y_S \in \{0,1\}^{|\Sk|} \mid \A_{[C_S(\Sk), \Sk]} \y_S \leq \b_{[C_S(\Sk)]}\right\}.
    \end{align*}
\end{definition}

Notice that $\cY \subseteq \cY_{\text{relax}}(\Sk)$, and equality holds if and only if $C_\partial(\Sk) = \emptyset$. Thus, the relaxed space strictly encompasses the true feasible set, allowing us to bound the local consistency gap by measuring how often the unconstrained, relaxed updates violate the dropped cut constraints.

\begin{proposition}[Cut-bounded consistency gap]
\label{prop:cut_bound}
    Assume the noise distribution $\cD$ has mutually independent coordinates. For any frozen coordinates $\Sk \!\subseteq\! [d]$ with $\Fk \!\triangleq\! [d]\setminus\Sk$, let $W_c(\Sk) \!\subseteq \!\cY_{\text{relax}}(\Sk)$ be the set of relaxed solutions violating the cut constraint $c \!\in \!C_\partial(\Sk)$:
    \begin{align*}
        W_c(\Sk) \triangleq \left\{ \y \in \cY_{\text{relax}}(\Sk) \mid \A_{[\{c\}, \Fk]} \y_{[\Fk]} + \A_{[\{c\}, \Sk]} \y_{[\Sk]} > b_c \right\},
    \end{align*}
    and define the global and relaxed perturbed maximizers for $Z \sim \cD$ as:
    \begin{align*}
        Y^* \triangleq \argmax_{\y \in \cY} \langle \thetav + \varepsilon Z, \y \rangle, \quad\text{and}\quad Y^*_{\text{relax}}(\Sk) \triangleq \argmax_{\y \in \cY_{\text{relax}}(\Sk)} \langle \thetav + \varepsilon Z, \y \rangle.
    \end{align*}
    Denoting $W_\partial(\Sk) \triangleq \bigcup_{c \in C_\partial(\Sk)} W_c(\Sk) = \cY_{\text{relax}}(\Sk) \setminus \cY$, we have:
    \begin{align*}
        \gape \leq \max_{\yk, \Sk} &\Bigg[ \PP\big( Y^*_{\text{relax}}(\Sk) \in W_\partial(\Sk) \;\big|\; (Y^*_{\text{relax}}(\Sk))_{[\Sk]} = \yk_{[\Sk]} \big)\\
        &\quad+ \PP\big( Y^*_{\text{relax}}(\Sk) \in W_\partial(\Sk) \;\big|\; Y^*_{[\Sk]} = \yk_{[\Sk]} \big) \Bigg].
    \end{align*}
    Furthermore, applying a union bound over $c \in C_\partial(\Sk)$ yields:
    \begin{align*}
        \gape &\leq \max_{\yk, \Sk} \Bigg\{ |C_\partial(\Sk)| \cdot\\
        &\max_{c \in C_\partial(\Sk)} \Big[ \PP\big( Y^*_{\text{relax}}(\Sk) \in W_c(\Sk) \;\big|\; (Y^*_{\text{relax}}(\Sk))_{[\Sk]} = \yk_{[\Sk]} \big)\\
        &\qquad\qquad\;+ \PP\big( Y^*_{\text{relax}}(\Sk) \in W_c(\Sk) \;\big|\; Y^*_{[\Sk]} = \yk_{[\Sk]} \big) \Big] \Bigg\}.
    \end{align*}
\end{proposition}

\paragraph{Motivation for min-cut variable selection.}
\cref{prop:cut_bound}, proved in \cref{proof:cut_bound}, provides a justification for constraint-graph heuristics in LNS. The upper bound on the bias of the local update scales with $|C_\partial(\Sk)|$, the number of cut constraints separating the frozen variables $\Sk$ from the free variables $\Fk$. This motivates the design of variable selection schemes that minimize $|C_\partial(\Sk)|$ in RLNS. In \cref{sec:mwis}, we empirically validate this theoretical insight on the maximum weight independent set problem.

\begin{figure}[h!]
    \centering
    \begin{subfigure}[b]{0.99\textwidth}
        \centering
        \includegraphics[width=\textwidth]{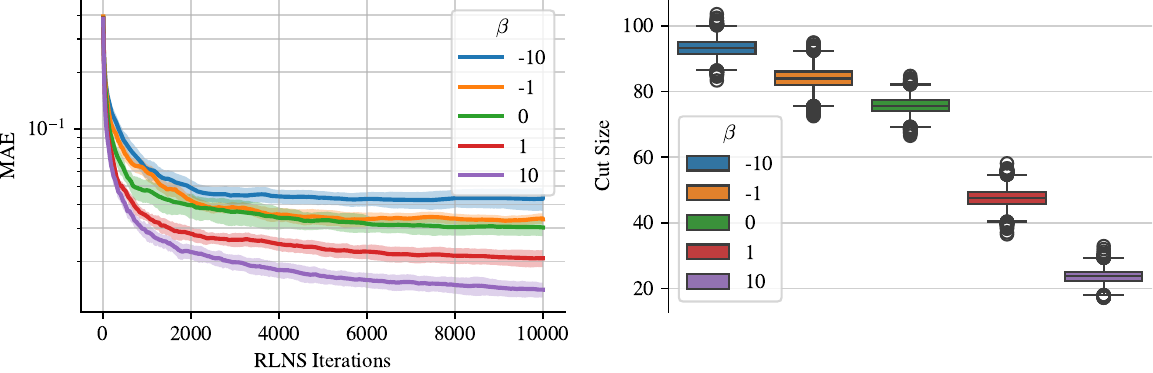}
    \end{subfigure}
    \vfill
    \begin{subfigure}[b]{0.99\textwidth}
        \centering
        \includegraphics[width=\textwidth]{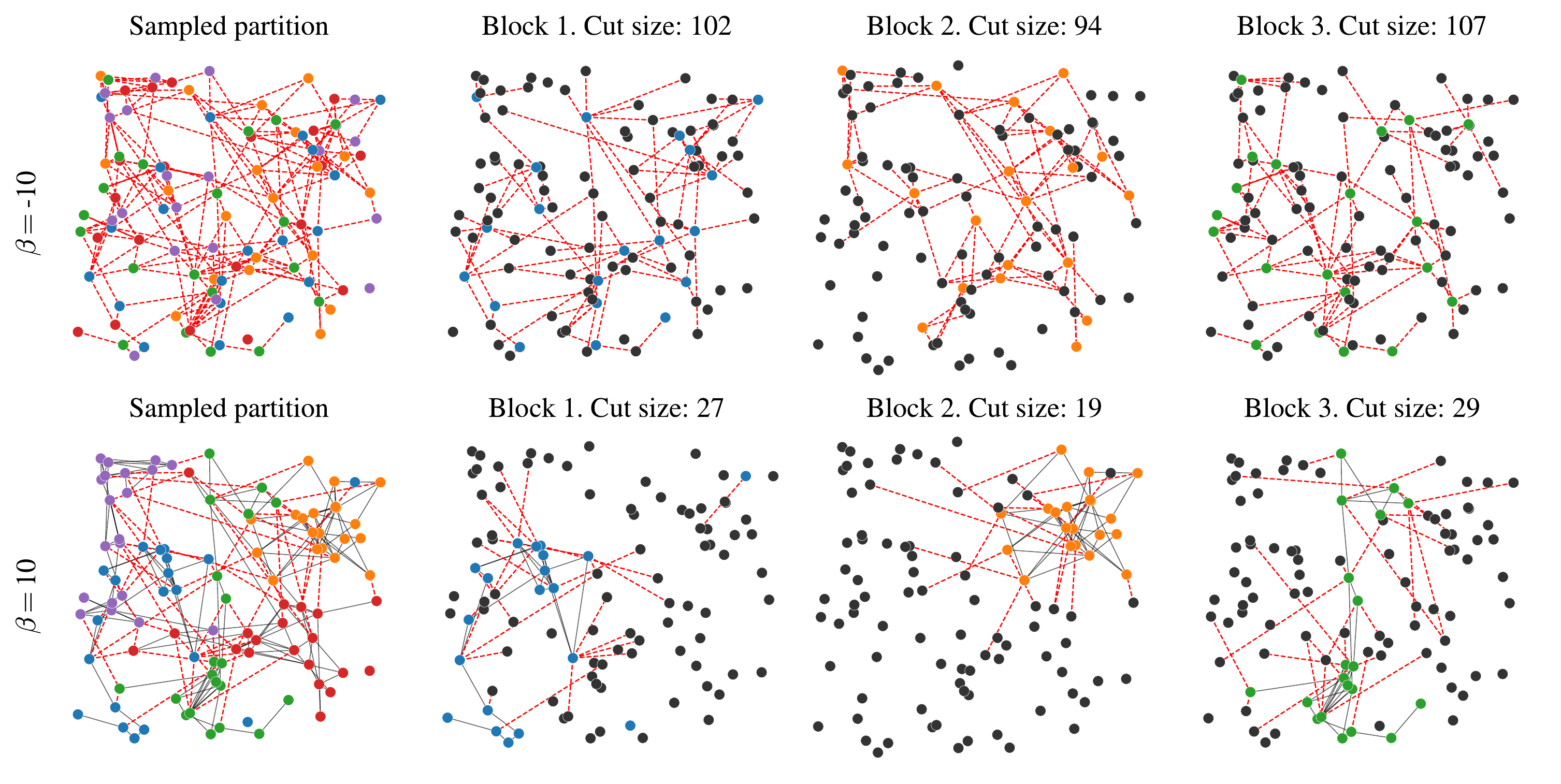}
    \end{subfigure}
    \caption{\textbf{Top left.} MAE between $K$-th RLNS estimate and true global expectation, computed as $\frac{1}{|\cN|}\|\frac{1}{K}\sum_{k=1}^K\yk-\EE_{\argmaxoe(\st)}[Y]\|_1$. \textbf{Top right.} Average cut size of the selected variables with different values of $\beta$. \textbf{Bottom.} Visualization of a sampled partition for $\beta=-10$ and $\beta=10$. Internal edges are represented by black lines, and cut edges are represented by red dashed lines.}
    \label{fig:mwis}
\end{figure}

\subsection{Min-cut RLNS on the maximum weight independent set}
\label{sec:mwis}

To empirically validate the cut-size-based bias bound from \cref{sec:constraint_graph}, we evaluate RLNS with perturbation-based regularization on the maximum weight independent set (MWIS) problem. Given an undirected graph $\cG=(\cN,\cE)$, the MWIS problem writes:

\begin{align*}
    \max_{\y\in\{0,1\}^{|\cN|}}&\sum_{n\in\cN}\theta_n y_n\\
    \text{s.t.} \quad & y_u+y_v\leq 1\quad\forall (u,v)\in\cE.
\end{align*}

\paragraph{Constraint graph.} For the MWIS, the constraint hypergraph $\cH_C$ exactly reduces to the standard problem graph $\cG$: since each constraint involves exactly two nodes, every hyperedge has size 2 and corresponds directly to an edge in $\cG$. Specifically, we generate soft random geometric graphs by randomly placing $|\cN|=100$ nodes in a 2D space and assigning edges with a probability that decays exponentially with the Euclidean distance between nodes. To test the impact of the cut size isolated in \cref{prop:cut_bound}, we design variable selection schemes with varying average cut sizes.

\paragraph{Variable selection scheme.} For each graph instance, we perform an offline phase where we generate a pool of $100$ node partitions, which all divide $\cN$ into 5 blocks of 20 nodes each. These partitions are generated via $10^4$ steps of Metropolis-Hastings MCMC (where a proposal partition is obtained by exchanging the block labels of two random nodes) to be sampled approximately proportionally to $\exp(-\beta\cdot c)$, where the cut size $c$ is the total number of edges $(u,v)\in\cE$ such that $u$ and $v$ belong to different blocks in the partition. By adjusting the inverse temperature $\beta$, we can smoothly interpolate the variable selection scheme from favoring max-cuts ($\beta<0$) to min-cuts ($\beta>0$): see \cref{fig:mwis} (bottom) for a visualization of sampled partitions. During the online RLNS procedure, at each step, we uniformly sample a partition from the offline pool and pick exactly one of its blocks uniformly at random to serve as the free variables, keeping the assignments of all other nodes frozen.

\paragraph{Results.} We report results averaged across 10 random seeds, where each seed controls the generation of a distinct random graph instance. As illustrated in \cref{fig:mwis} (top), the results validate our theoretical insights from \cref{sec:constraint_graph}: variable selection schemes with minimal cut size $|C_\partial(\Sk)|$ in the constraint graph yield a better convergence to the exact global regularized distribution $\argmaxoe(\st)$.

\section{Proofs}
\label{sec:proofs}

\subsection{Proof of \cref{prop:gibbs_mcmc} (\nameref{prop:gibbs_mcmc})}

\begin{proof}
\label{proof:gibbs_mcmc}

When $\Omega=-\gamma H^s$, we have, for any $D\in\NN,\,\x\in\RR^D$:
\begin{align*}
    \argmaxo(\x)=\left(\frac{\exp(x_i/\gamma)}{\sum_{j=1}^D\exp(x_j/\gamma)}\right)_{i=1}^D.
\end{align*}
Thus, on the one hand, the local regularized distribution writes, for all $\yk\in\cY$ and $\y\in\Yk$:
\begin{align*}
    \argmaxok\left(\stk\right)(\y) = \frac{\exp\left(\langle\thetav,\y\rangle/\gamma\right)}{\sum_{\y'\in\Yk}\exp(\langle\thetav,\y'\rangle/\gamma)}.
\end{align*}
On the other hand, the conditional probability of the global regularized distribution is given by:
\begin{align*}
    \argmaxo(\st)\left(\y\mid\yk_{[\Sk]}\right) &\triangleq \frac{\argmaxo(\st)(\y)}{\sum_{\substack{\y'\in\cY\\\text{s.t.}\y'_{[\Sk]}=\yk_{[\Sk]}}}\argmaxo(\st)(\y')}\\
    &= \frac{\argmaxo(\st)(\y)}{\sum_{\y'\in\Yk}\argmaxo(\st)(\y')}\\
    &= \frac{\exp(\langle\thetav,\y\rangle/\gamma)}{\sum_{\y''\in\cY}\exp(\langle\thetav,\y''\rangle/\gamma)}\times\left(\sum_{\y'\in\Yk}\frac{\exp(\langle\thetav,\y'\rangle/\gamma)}{\sum_{\y''\in\cY}\exp(\langle\thetav,\y''\rangle/\gamma)}\right)^{-1}\\
    &= \frac{\exp(\langle\thetav,\y\rangle/\gamma)}{\sum_{\y''\in\cY}\exp(\langle\thetav,\y''\rangle/\gamma)}\times\frac{\sum_{\y''\in\cY}\exp(\langle\thetav,\y''\rangle/\gamma)}{\sum_{\y'\in\Yk}\exp(\langle\thetav,\y'\rangle/\gamma)}\\
    &= \frac{\exp(\langle\thetav,\y\rangle/\gamma)}{\sum_{\y'\in\Yk}\exp(\langle\thetav,\y'\rangle/\gamma)},
\end{align*}
which therefore gives:
\begin{align*}
    \argmaxok\left(\stk\right)(\y)=\argmaxo(\st)\left(\y\mid\yk_{[\Sk]}\right).
\end{align*}

Thus, \cref{algo:rlns} implements a block Gibbs sampler for the target distribution $\pi \triangleq \argmaxo(\st)$. To check that detailed balance is satisfied, let $\nu(\Sk)$ denote the probability of selecting the frozen subset $\Sk$. The transition probability from $\yk$ to $\y$ is:
\begin{align*}
    P(\yk \to \y) = \sum_{\Sk \subseteq [d]} \nu(\Sk) \mathds{1}_{\{\y_{[\Sk]} = \yk_{[\Sk]}\}} \pi\left(\y \mid \yk_{[\Sk]}\right).
\end{align*}
Multiplying by the stationary distribution candidate $\pi(\yk)$, we get:
\begin{align*}
    \pi(\yk) P(\yk \to \y) &= \sum_{\Sk \subseteq [d]} \nu(\Sk) \mathds{1}_{\{\y_{[\Sk]} = \yk_{[\Sk]}\}} \pi(\yk) \frac{\pi(\y)}{\sum_{\z \in \Yk} \pi(\z)} \\
    &= \sum_{\Sk \subseteq [d]} \nu(\Sk) \mathds{1}_{\{\y_{[\Sk]} = \yk_{[\Sk]}\}} \frac{\pi(\yk) \pi(\y)}{\sum_{\z:\, \z_{[\Sk]} = \yk_{[\Sk]}} \pi(\z)}.
\end{align*}
Because the indicator function enforces $\y_{[\Sk]} = \yk_{[\Sk]}$, the denominator is strictly identical whether we condition on $\yk_{[\Sk]}$ or $\y_{[\Sk]}$. The expression inside the sum is therefore perfectly symmetric with respect to $\yk$ and $\y$, which immediately yields:
\begin{align*}
    \pi(\yk) P(\yk \to \y) = \pi(\y) P(\y \to \yk),
\end{align*}
so that $\argmaxo(\st)$ is a stationary distribution of the Markov chain. Finally, an irreducible variable selection scheme ensures the chain is ergodic (since the strictly positive entropic update distributions guarantee dense transitions and aperiodicity via self-loops). This gives that the stationary distribution is uniquely equal to $\piLNS=\argmaxo(\st)$, and that the chain converges to it.
\end{proof}

\subsection{Proof of \cref{prop:stationary_equivalence} (\nameref{prop:stationary_equivalence})}

\begin{proof}
\label{proof:stationary_equivalence}
The sense ($\impliedby$) is established by \cref{prop:gibbs_mcmc}, so we focus on the ($\implies$) implication. 

Assume that $\Omega:\triangleY\to\RR$ is separable and Legendre-type on $\triangleY$ \citep{wainwright_graphical_2008,blondel_learning_2020}, meaning $\Omega(\p) = \sum_{\y\in\cY} \omega(p_\y)$ for some l.s.c. strictly convex function $\omega:[0,1]\to\RR$ that is differentiable on $(0,1)$ with $\lim_{p\to 0^+}\omega'(p)=-\infty$ (so that $\argmaxo$ maps $\RRY$ onto $\ri(\triangleY)$ and local update probabilities are strictly positive). Since $\Omega(\delta^\cY_\y)=0\;\;\forall\y\in\cY$, we assume standard simplex normalization $\omega(0)=\omega(1)=0$ without loss of generality. We assume that $\piLNS = \argmaxo(\st)$ for all score vectors $\st$, and we will show this forces $\Omega = -\gamma H^s$ for some $\gamma>0$.

\paragraph{Problem definition.}
Consider a state space of dimension $d=3$ with feasible set $\cY=\{\e_1,\e_2,\e_3\}$. Let $\thetav=(\theta_1,\theta_2,\theta_3)\in\RR^3$ be any maximization direction. Note that in this setting, we have $\st=\thetav$. We instantiate the following specific variable selection strategy for \cref{algo:rlns}: at each step $k$, given a current solution $\yk=\e_i$, the algorithm samples a frozen coordinate subset $\Sk \in \{\{1\}, \{3\}\}$ uniformly at random. We strictly prevent freezing coordinate $2$.

This choice restricts the possible transitions. For instance, if $\yk=\e_1$ (meaning $y^{(k)}_1=1$), freezing coordinate $3$ (where $y^{(k)}_3=0$) creates the local neighborhood $\Yk = \{\e_1, \e_2\}$. However, it is impossible to transition directly between $\e_1$ and $\e_3$ in a single step, as that would require freezing coordinate $2$, which our algorithm never does. The state transition graph on $\cY$ is therefore a simple line graph: $\e_1 \leftrightarrow \e_2 \leftrightarrow \e_3$.

\paragraph{Detailed balance.}
Because the transition graph is a tree (it contains no cycles), any stationary distribution on this graph must satisfy detailed balance. Since we assume the RLNS stationary distribution matches the global target ($\piLNS = \argmaxo(\st)$), the target distribution $\argmaxo(\st)$ itself must satisfy detailed balance along the edges of the tree. Let $\piv\triangleq\argmaxo(\st)$, i.e., $\pi_i \triangleq \argmaxo(\st)(\e_i)$. The balance equation between $\e_1$ and $\e_2$ is:
\begin{align*}
    \pi_1 \cdot \PP(\e_1 \to \e_2) = \pi_2 \cdot \PP(\e_2 \to \e_1).
\end{align*}
The transition from $\e_1$ to $\e_2$ (resp. from $\e_2$ to $\e_1$) occurs by choosing to freeze coordinate $3$ (which occurs with probability $1/2$) and then sampling $\e_2$ (resp. $\e_1$) from the local distribution $q^{\{1,2\}}\triangleq \argmaxok(\stk)$ (where $\Yk=\{\e_1,\e_2\}$). To simplify notation, let $q_1 \triangleq q^{\{1,2\}}(\e_1)$ and $q_2 \triangleq q^{\{1,2\}}(\e_2)$. We have:
\begin{align*}
    \pi_1 \cdot\left( \frac{1}{2} q_2 \right) = \pi_2 \cdot\left( \frac{1}{2} q_1 \right).
\end{align*}
Since local probabilities sum to $1$, substituting $q_2= 1 - q_1$ immediately yields:
\begin{align}
\label{eq:local_global_ratio}
    q_1 = \frac{\pi_1}{\pi_1 + \pi_2} \quad \text{and} \quad q_2 = \frac{\pi_2}{\pi_1 + \pi_2}.
\end{align}

\paragraph{Optimality conditions.}
Because $\Omega(\p) = \sum_{i=1}^3 \omega(p_i)$ is strictly convex and differentiable on the interior of the simplex, the global distribution $\piv=\argmaxo(\st)$ is by definition the unique solution to the following constrained optimization problem:
\begin{align*}
    \min_{\p > 0} \quad & \sum_{i=1}^3 \big(\omega(p_i) - \theta_i p_i\big) \\
    \text{s.t.} \quad & \sum_{i=1}^3 p_i = 1.
\end{align*}
We introduce a Lagrange multiplier $\nu \in \RR$ for the simplex equality constraint. The Lagrangian is given by:
\begin{align*}
    \mathcal{L}(\p, \nu) = \sum_{i=1}^3 \big(\omega(p_i) - \theta_i p_i\big) - \nu \left( \sum_{i=1}^3 p_i - 1 \right).
\end{align*}
The Karush-Kuhn-Tucker (KKT) stationarity conditions require the partial derivatives of the Lagrangian with respect to each $p_i$ to be zero at optimality:
\begin{align*}
    \frac{\partial \mathcal{L}}{\partial p_i}(\piv,\nu^\star) = \omega'(\pi_i) - \theta_i - \nu^\star = 0 \implies \omega'(\pi_i) = \theta_i + \nu^\star.
\end{align*}
By subtracting the optimality condition of $i=2$ from $i=1$, the multiplier $\nu^\star$ cancels out, yielding:
\begin{align}
\label{eq:kkt_global}
    \omega'(\pi_1) - \omega'(\pi_2) = \theta_1 - \theta_2.
\end{align}

Similarly, the local distribution $q^{\{1,2\}}$ solves by definition the local constrained problem:
\begin{align*}
    \min_{q_1, q_2 > 0} \quad & \big(\omega(q_1) - \theta_1 q_1\big) + \big(\omega(q_2) - \theta_2 q_2\big) \\
    \text{s.t.} \quad & q_1 + q_2 = 1.
\end{align*}
Note that the $\omega(0)$ term due to padding the local distribution $\q$ into a global one (see the definition of the local regularizer in \cref{def:local_reg}) vanishes since we assumed $\omega(0)=\omega(1)=0$.
Introducing a local Lagrange multiplier $\mu \in \RR$, the local Lagrangian is:
\begin{align*}
    \mathcal{L}_{\text{local}}(\q, \mu) = \big(\omega(q_1) - \theta_1 q_1\big) + \big(\omega(q_2) - \theta_2 q_2\big) - \mu (q_1 + q_2 - 1).
\end{align*}
The KKT stationarity conditions for the local problem similarly yield $\omega'(q_i) = \theta_i + \mu^\star$. Subtracting the condition for $q_2$ from $q_1$ gives:
\begin{align}
\label{eq:kkt_local}
    \omega'(q_1) - \omega'(q_2) = \theta_1 - \theta_2.
\end{align}

Combining \cref{eq:kkt_global} and \cref{eq:kkt_local}, we get:
\begin{align*}
    \omega'(\pi_1) - \omega'(\pi_2) = \omega'(q_1) - \omega'(q_2).
\end{align*}
Finally, substituting the expressions of $q_1$ and $q_2$ derived in \cref{eq:local_global_ratio}, we get:
\begin{align}
\label{eq:derivative_eq}
    \omega'(\pi_1) - \omega'(\pi_2) = \omega'\left(\frac{\pi_1}{\pi_1+\pi_2}\right) - \omega'\left(\frac{\pi_2}{\pi_1+\pi_2}\right).
\end{align}

\paragraph{Functional equation.}
Let $y \triangleq \pi_1 + \pi_2$ and $x \triangleq \frac{\pi_1}{\pi_1+\pi_2}$, such that $\pi_1 = xy$ and $\pi_2 = (1-x)y$. Since the global scores $\theta_1, \theta_2, \theta_3$ are unconstrained in $\RR$ and $\argmaxo$ is surjective onto the relative interior $\ri(\triangleY)$ (as $\argmaxo=\nabla\Omega^*$ and $\Omega$ is Legendre-type), $\piv$ spans $\ri(\triangleY)$, so we have that $(\pi_1,\pi_2)$ spans $\{(\pi_1,\pi_2)\in(0,1)^2 \mid \pi_1+\pi_2<1\}$ and thus $x$ and $y$ independently span $(0,1)$.

Substituting $\pi_1 = xy$ and $\pi_2 = (1-x)y$ into \cref{eq:derivative_eq} yields:
\begin{align*}
    \omega'(xy) - \omega'((1-x)y) = \omega'(x) - \omega'(1-x).
\end{align*}
We integrate this equation with respect to $x$ (holding $y$ constant):
\begin{align*}
    \int \left[ \omega'(xy) - \omega'((1-x)y) \right] dx &= \int \left[ \omega'(x) - \omega'(1-x) \right] dx \\
    \implies \frac{1}{y}\omega(xy) + \frac{1}{y}\omega((1-x)y) &= \omega(x) + \omega(1-x) + C(y).
\end{align*}
Multiplying by $y$, we get:
\begin{align}
\label{eq:almost_functional}
    \omega(xy) + \omega((1-x)y) = y(\omega(x) + \omega(1-x)) + y C(y).
\end{align}
To determine the integration term $y C(y)$, we evaluate the limit as $x \to 1$. Because we normalized $\omega(0)=\omega(1)=0$, taking $x \to 1$ yields:
\begin{align*}
    \omega(y) + \omega(0) = y(\omega(1) + \omega(0)) + y C(y) \implies \omega(y) = y C(y).
\end{align*}
Substituting $yC(y) = \omega(y)$ back into \cref{eq:almost_functional} yields the functional equation:
\begin{align*}
    \omega(xy) + \omega((1-x)y) - \omega(y) = y(\omega(x) + \omega(1-x)) \quad \forall\; x, y \in (0,1).
\end{align*}
As first established in \citet[Theorem 0]{horibe1988entropy} and extended in \citet{gselmann2011entropy}, all measurable solutions to this functional equation take the form $\omega(p)=\gamma p\log p+ cp$ for some constants $\gamma, c\in\RR$. Since we assumed $\omega(1)=0$, we have $c=0$. Moreover, since we assumed $\omega$ to be strictly convex, we get $\gamma > 0$, which concludes the proof.
\end{proof}

\subsection{Strict convexity of $\Omegae$}
\label{proof:strict_convexity}

In this section, we establish that the perturbation-based regularizer $\Omegae$ defined in \cref{def:perturbation_regularization} is strictly convex over the entire closed simplex $\triangleY$.

\begin{proposition}[Strict convexity of $\Omegae$]
\label{prop:strict_convexity_omegae}
Let $\Omegae : \triangleY \to \RR$ be the perturbation-based regularization function defined in \cref{def:perturbation_regularization}. If the noise distribution $\cD$ admits a strictly positive density on $\RR^d$, then $\Omegae$ is strictly convex on the entire closed simplex $\triangleY$.
\end{proposition}

\begin{proof}
Let $\p, \q \in \triangleY$ be two distinct probability distributions, and let $\lambda \in (0, 1)$. We define their convex combination $\r \triangleq \lambda \p + (1-\lambda) \q$. The support of $\r$ is the union of the supports of $\p$ and $\q$, i.e., $\cZ \triangleq \supp(\r) = \supp(\p) \cup \supp(\q) \subseteq \cY$.

Consider the local simplex $\triangle^\cZ$ and the inclusion map $\iota_\cZ : \triangle^\cZ \lhook\joinrel\to \triangleY$ as defined in \cref{def:local_reg}. By construction, $\p$ and $\q$ belong to the face $\iota_\cZ(\triangle^\cZ)$, and $\r$ lies strictly in the relative interior of this face, meaning $\r \in \ri(\iota_\cZ(\triangle^\cZ))$.

To prove strict convexity, we must show that:
\begin{align*}
    \Omegae(\r) < \lambda \Omegae(\p) + (1-\lambda) \Omegae(\q).
\end{align*}
Assume for contradiction that equality holds. Since $\Omegae$ is a supremum of affine functions, it is convex and lower semi-continuous. The equality at an interior point of the segment $[\p, \q]$ implies that $\Omegae$ must be strictly affine on the entire segment $[\p, \q]$. Because $\r \in \ri(\iota_\cZ(\triangle^\cZ))$, there exists a small line segment centered at $\r$, contained entirely within $[\p, \q]$, that lies completely within the relative interior $\ri(\iota_\cZ(\triangle^\cZ))$. If $\Omegae$ is affine on $[\p, \q]$, it must be affine on this sub-segment.

We now analyze the restriction of $\Omegae$ to this face. As shown in \cref{proof:perturbation_local}, the local regularizer $\Omegae^\cZ \triangleq \Omegae \circ \iota_\cZ$ is exactly the Fenchel conjugate of the local perturbed maximum function $F_{\varepsilon, \cZ} : \RR^{|\cZ|} \to \RR$:
\begin{align*}
    F_{\varepsilon, \cZ}(\u) \triangleq \EE_{Z\sim\cD}\left[ \max_{\y \in \cZ} \left\{ u_\y + \varepsilon \langle \y, Z \rangle \right\} \right].
\end{align*}

By \citet[Theorem 4]{bouvier_primal-dual_2025}, we have that:
\begin{enumerate}
    \item The domain of the conjugate of $F_{\varepsilon, \cZ}$ is $\triangle^\cZ$, i.e., $\dom(\Omegae^\cZ)=\triangle^\cZ$.
    \item $F_{\varepsilon, \cZ}$ is l.s.c., proper, convex and differentiable on $\RR^{|\cZ|}$, which makes it essentially smooth.
\end{enumerate}

\citet[Theorem 26.3]{rockafellar_convex_1970} then gives that its Fenchel conjugate $\Omegae^\cZ$ is essentially strictly convex. Therefore, it is strictly convex on the relative interior of its domain, $\ri(\dom(\Omegae^\cZ))=\ri(\triangle^\cZ)$.

This establishes that $\Omegae$ is strictly convex on $\ri(\iota_\cZ(\triangle^\cZ))$. However, we previously deduced that $\Omegae$ would be affine on a sub-segment passing through $\r$ within this relative interior, which is a contradiction. Therefore, the strict inequality must hold, concluding the proof that $\Omegae$ is strictly convex on the entire closed simplex $\triangleY$.
\end{proof}

\subsection{Proof of \cref{prop:perturbation_local} (\nameref{prop:perturbation_local})}
\label{proof:perturbation_local}

\begin{proof}
    Let $F_{\varepsilon, \Yk}$ denote the following local perturbed maximum function, defined on scores restricted to the neighborhood $\Yk$:
    \begin{align*}
        F_{\varepsilon, \Yk}: \RR^{|\Yk|} &\longrightarrow \RR \\
        \u &\longmapsto \EE_{Z\sim\cD}\left[ \max_{\y \in \Yk} \left\{ u_\y + \varepsilon\langle\y,Z\rangle \right\} \right].
    \end{align*}
    We first establish that the local regularization function $\Omegae^{\Yk} \triangleq \Omegae \circ \iotak$ is the Fenchel conjugate of $F_{\varepsilon, \Yk}$.
    Starting from the definition of $\Omegae$, we have for any local distribution $\q \in \trianglek$:
    \begin{align*}
        \Omegae^{\Yk}(\q) &= \Omegae(\iotak(\q)) \\
        &= \sup_{\s \in \RRY} \left\{\langle \s, \iotak(\q) \rangle - \EE_{Z\sim\cD}\left[ \max_{\y' \in \cY} \left\{s_{\y'} + \varepsilon \langle\y', Z\rangle\right\} \right] \right\}\\
        &= \sup_{\s \in \RRY} \left\{\left( \sum_{\y \in \Yk} s_\y q_\y\right) - \EE_{Z\sim\cD}\left[ \max_{\y' \in \cY} \left\{s_{\y'} + \varepsilon \langle\y', Z\rangle\right\} \right] \right\}.
    \end{align*}
    We decompose the global score vector $\s$ into components inside $\Yk$ (denoted $\s_{\mathrm{in}}$) and outside $\Yk$ (denoted $\s_{{\mathrm{out}}}$). To maximize the objective inside the supremum, we observe that $\s_{{\mathrm{out}}}$ only appears in the negative expectation term. Sending the coordinates of $\s_{{\mathrm{out}}}$ to $-\infty$ minimizes the expectation term (by eliminating $\y' \notin \Yk$ from the maximum) without affecting the linear term $\langle \s_{{\mathrm{in}}}, \q \rangle$. 
    
    Formally, we have:
    \begin{align*}
        \lim_{\substack{s_{\y'} \to -\infty\\ \forall \y'\in\cY\setminus\Yk}} \EE_{Z\sim\cD}\left[ \max_{\y \in \cY} \left\{s_{\y} + \varepsilon \langle\y,Z\rangle\right\} \right] &= \EE_{Z\sim\cD}\left[ \max_{\y \in \Yk} \left\{(\s_{\mathrm{in}})_{\y} + \varepsilon \langle\y,Z\rangle\right\} \right]\\
        &\leq \EE_{Z\sim\cD}\left[ \max_{\y \in \cY} \left\{s_{\y} + \varepsilon \langle\y, Z\rangle\right\} \right]\;\forall \s\in\RRY.
    \end{align*}
    Substituting this back into the conjugate expression:
    \begin{align*}
        \Omegae^{\Yk}(\q) &= \sup_{\s_{{\mathrm{in}}} \in \RR^{|\Yk|}} \left\{ \langle \s_{{\mathrm{in}}}, \q \rangle - \EE_{Z\sim\cD}\left[ \max_{\y \in \Yk} \left\{(\s_{{\mathrm{in}}})_{\y} + \varepsilon \langle\y,Z\rangle\right\} \right] \right\} \\
        &= \sup_{\u \in \RR^{|\Yk|}} \left\{ \langle \u, \q \rangle - F_{\varepsilon, \Yk}(\u) \right\} \\
        &= F_{\varepsilon, \Yk}^*(\q).
    \end{align*}
    Thus, the local regularizer is exactly the conjugate of the local perturbed maximum function. 
    
    Finally, we compute the local regularized distribution. By definition:
    \begin{align*}
        \argmaxoek\left(\stk\right) &\triangleq \argmax_{\q \in \trianglek} \left\{ \langle \stk, \q \rangle - \Omegae^{\Yk}(\q) \right\} \\
        &= \argmax_{\q \in \trianglek} \left\{ \langle \stk, \q \rangle - F_{\varepsilon, \Yk}^*(\q) \right\},
    \end{align*}
    so that Danskin's theorem gives:
    \begin{align*}
        \argmaxoek\left(\stk\right) = \nabla F_{\varepsilon, \Yk}^{**}\left(\stk\right).
    \end{align*}
    Since $F_{\varepsilon, \Yk}$ is closed proper convex, it is equal to its biconjugate by the Fenchel-Moreau theorem, and we have in fact:
    \begin{align*}
        \argmaxoek\left(\stk\right) = \nabla F_{\varepsilon, \Yk}\left(\stk\right).
    \end{align*}
    We then invert the differentiation and expectation signs (this application of the \emph{path-wise gradient estimator}, or the \emph{reparameterization trick}, is possible since the expectation is taken with respect to $Z$ which does not depend on $\st$ \citep{blondel_elements_2025}), and further apply Danskin's theorem a second time:
    \begin{align*}
        \nabla F_{\varepsilon, \Yk}\left(\stk\right) &= \nabla_{\stk} \EE_{Z\sim\cD}\left[ \max_{\y \in \Yk} \left\{\left(\stk\right)_\y + \varepsilon \langle\y,Z\rangle \right\} \right] \\
        &= \EE_{Z\sim\cD}\left[ \nabla_{\stk} \max_{\y \in \Yk} \left\{ (\stk)_\y + \varepsilon \langle\y,Z\rangle \right\} \right] \\
        &= \EE_{Z\sim\cD}\left[ \delta^{\Yk}_{\argmax_{\y \in \Yk} \langle \thetav + \varepsilon Z, \y \rangle} \right].
    \end{align*}
\end{proof}

\subsection{Proof of \cref{prop:mean_space_bca} (\nameref{prop:mean_space_bca})}
\label{proof:mean_space_bca}

\begin{proof}
We first establish the strict convexity of the mean-space regularizer $\Omega_\cM$.
Let $\muv_1, \muv_2 \in \cM$ such that $\muv_1 \neq \muv_2$, and let $\lambda \in (0, 1)$. Because the simplex $\triangleY$ is compact and $\Omega$ is strictly convex, the minimum in the definition of $\Omega_\cM$ is uniquely attained for any $\muv \in \cM$. Let $\p_1$ and $\p_2$ be the unique minimizers corresponding to $\muv_1$ and $\muv_2$, respectively. Since $\muv_1 \neq \muv_2$, we must have $\p_1 \neq \p_2$.

Define the convex combinations $\muv_\lambda \triangleq \lambda \muv_1 + (1-\lambda)\muv_2$ and $\p_\lambda \triangleq \lambda \p_1 + (1-\lambda)\p_2$. By linearity of expectation, we have $\EE_{\p_\lambda}[Y] = \muv_\lambda$. Since $\p_\lambda \in \triangleY$, it is a valid candidate distribution for the constrained minimization problem defining $\Omega_\cM(\muv_\lambda)$. Therefore, we have $\Omega_\cM(\muv_\lambda) \leq \Omega(\p_\lambda)$. Because $\Omega$ is strictly convex and $\p_1 \neq \p_2$, we strictly bound the right-hand side:
\begin{align*}
    \Omega(\p_\lambda) < \lambda \Omega(\p_1) + (1-\lambda)\Omega(\p_2).
\end{align*}
Substituting the definitions of the marginal minima yields $\Omega_\cM(\muv_\lambda) < \lambda \Omega_\cM(\muv_1) + (1-\lambda)\Omega_\cM(\muv_2)$, which proves that $\Omega_\cM$ is strictly convex on $\cM$.

We now turn to the BCA step. Let $\Mk \triangleq \{ \muv \in \cM \mid \muv_{[\Sk]} = \yk_{[\Sk]} \}$ denote the feasible set of the mean-space BCA step. We must show that this slice of the marginal polytope strictly corresponds to the convex hull of the local neighborhood $\Yk \triangleq \{ \y \in \cY \mid \y_{[\Sk]} = \yk_{[\Sk]} \}$.

Clearly, $\conv(\Yk) \subseteq \Mk$. We must prove $\Mk \subseteq \conv(\Yk)$.
Let $\muv \in \Mk$. Because $\muv \in \cM$, there exists a distribution $\p \in \triangleY$ such that $\muv = \sum_{\y \in \cY} p_\y \y$.
For any coordinate $i \in \Sk$, we are given that $\cY \subseteq \{0,1\}^d$, so $y_i \in \{0,1\}$ for all $\y \in \cY$, and $y^{(k)}_i \in \{0,1\}$.
Since $\muv \in \Mk$, we have $\mu_i = y^{(k)}_i$. We distinguish two cases:
\begin{itemize}[topsep=0pt,itemsep=1.5pt,parsep=2pt,leftmargin=15pt]
    \item If $y^{(k)}_i = 0$, then $\mu_i = 0$, giving $\sum_{\y \in \cY} p_\y y_i = 0$. Since $p_\y \geq 0$ and $y_i \in \{0,1\}$, all terms in the sum are non-negative. This implies that for any $\y \in \cY$, $p_\y > 0 \implies y_i = 0$.
    \item If $y^{(k)}_i = 1$, then $\mu_i = 1$, giving $\sum_{\y \in \cY} p_\y y_i = 1$. Since $\sum_{\y \in \cY} p_\y = 1$, we can rewrite this as $\sum_{\y \in \cY} p_\y (1 - y_i) = 0$. Since $p_\y \geq 0$ and $1 - y_i \geq 0$, we similarly conclude that $p_\y > 0 \implies y_i = 1$.
\end{itemize}
In both cases, any solution $\y$ with non-zero probability mass ($p_\y > 0$) must strictly match $\yk$ on all coordinates $i \in \Sk$. Thus, $\supp(\p) \subseteq \Yk$, which proves that $\muv \in \conv(\Yk)$, and therefore $\Mk = \conv(\Yk)$.

Substituting the definition of the mean-space regularizer $\Omega_\cM$ into the exact BCA update, the next mean-space iterate $\muv^{(k+1)}$ is given by:
\begin{align*}
    \muv^{(k+1)} = \argmax_{\muv \in \Mk} \left[ \langle \thetav, \muv \rangle - \min_{\p \in \triangleY \text{ s.t. } \EE_\p[Y] = \muv} \Omega(\p) \right].
\end{align*}
Distributing the minus sign turns the inner minimum into a maximum:
\begin{align*}
    \muv^{(k+1)} = \argmax_{\muv \in \Mk} \left[ \max_{\p \in \triangleY \text{ s.t. } \EE_\p[Y] = \muv} \left( \langle \thetav, \muv \rangle - \Omega(\p) \right) \right].
\end{align*}
Using the fact that $\langle \thetav, \muv \rangle = \langle \thetav, \EE_\p[Y] \rangle = \langle \st, \p \rangle$ via the definition of the score vector $\st \triangleq (\langle \thetav, \y \rangle)_{\y \in \cY}$, we can absorb the maximization over $\muv$ and directly optimize over the distribution $\p$:
\begin{align*}
    \muv^{(k+1)} &= \EE_{\p^*}[Y], \quad \text{where} \quad \p^* = \argmax_{\substack{\p \in \triangleY \\ \EE_\p[Y]_{[\Sk]} = \yk_{[\Sk]}}} \left[ \langle \st, \p \rangle - \Omega(\p) \right].
\end{align*}
As established earlier, the constraint $\EE_\p[Y]_{[\Sk]} = \yk_{[\Sk]}$ strictly confines the support of $\p$ to $\Yk$. We can therefore reparameterize $\p$ using the local inclusion map $\iotak : \trianglek \hookrightarrow \triangleY$. Setting $\p = \iotak(\q)$ for $\q \in \trianglek$, the optimization problem restricts to the local simplex:
\begin{align*}
    \q^* = \argmax_{\q \in \trianglek} \left[ \langle \st, \iotak(\q) \rangle - \Omega(\iotak(\q)) \right].
\end{align*}
By the definition of the local regularizer $\Omegak \triangleq \Omega \circ \iotak$, and since $\langle \st, \iotak(\q) \rangle = \langle \stk, \q \rangle$, this simplifies exactly to the definition of the local smoothed maximizer:
\begin{align*}
    \q^* &= \argmax_{\q \in \trianglek} \left[ \langle \stk, \q \rangle - \Omegak(\q) \right] \\
    &= \argmaxok\left(\stk\right).
\end{align*}
Since $\muv^{(k+1)} = \EE_{\p^*}[Y] = \EE_{\q^*}[Y]$, we conclude that the exact mean-space BCA update yields exactly the expected value of the distribution sampled by the RLNS algorithm.
\end{proof}

\subsection{Proof of \cref{prop:updates_as_projections}  (\nameref{prop:updates_as_projections})}
\label{proof:updates_as_projections}

First, we prove that the local RLNS update distributions match the $B_\Omega$-projection of $\argmaxo(\st)$ onto the local simplex $\trianglek$.

\begin{proof}
    \textbf{Global optimization problem.}
    By definition, the global smoothed maximizer $\argmaxo(\st)$ is the unique solution to the convex optimization problem:
    \begin{align*}
        \argmaxo(\st) = \argmin_{\p \in \triangleY} \left\{ \Omega(\p) - \langle \st, \p \rangle \right\}.
    \end{align*}
    We introduce a Lagrange multiplier $\nu \in \RR$ for the simplex equality constraint $\sum_{\y \in \cY} p_\y = 1$. The Lagrangian is $\cL(\p, \nu) = \Omega(\p) - \langle \st, \p \rangle + \nu\cdot (\langle \p ,\1\rangle - 1)$. The first-order optimality condition gives:
    \begin{align}
    \label{eq:global_foc}
        \nabla_\p\cL(\argmaxo(\st),\nu^\star)=\0 \iff \nabla \Omega(\argmaxo(\st)) = \st - \nu^\star \1.
    \end{align}
    
    \textbf{Local optimization problem.}
    By definition of the local regularized distribution, $\argmaxok(\stk)$ is the solution to:
    \begin{align*}
        \argmaxok(\stk) &\triangleq \argmin_{\q \in \trianglek} \left\{ \Omegak(\q) - \langle \stk, \q \rangle \right\} \\
        &= \argmin_{\q \in \trianglek} \left\{ \Omega(\iotak(\q)) - \langle \st, \iotak(\q) \rangle \right\}.
    \end{align*}
    
    \textbf{Equivalence to projection.}
    We now analyze the projection problem. We minimize the Bregman divergence between the global solution $\argmaxo(\st)$ and a distribution $\q\in\trianglek$ restricted to the local simplex:
    \begin{align*}
        B_\Omega\left(\iotak(\q)\| \argmaxo(\st)\right) = \;&\Omega(\iotak(\q)) - \Omega(\argmaxo(\st))- \langle \nabla \Omega(\argmaxo(\st)), \iotak(\q) - \argmaxo(\st) \rangle.
    \end{align*}
    Substituting the global optimality condition from \cref{eq:global_foc} into the linear term:
    \begin{align*}
        \langle \nabla \Omega(\argmaxo(\st)), \iotak(\q) \rangle &= \langle \st - \nu^\star \bm{1}, \iotak(\q) \rangle \\
        &= \langle \st, \iotak(\q) \rangle - \nu^\star \sum_{\y \in \Yk} q_\y \\
        &= \langle \st, \iotak(\q) \rangle - \nu^\star \quad (\text{since } \q \in \trianglek).
    \end{align*}
    Thus, we have:
    \begin{align*}
        B_\Omega\left(\iotak(\q)\| \argmaxo(\st)\right) &= \Omega(\iotak(\q)) - \Omega(\argmaxo(\st)) - \langle \st, \iotak(\q) \rangle \\
        &\quad + \nu^\star + \langle \nabla \Omega(\argmaxo(\st)), \argmaxo(\st) \rangle.
    \end{align*}

    Removing terms that are constant with respect to $\q$, the projection problem reduces to:
    \begin{align*}
        \argmin_{\q \in \trianglek} B_\Omega\left(\iotak(\q)\| \argmaxo(\st)\right) = \argmin_{\q \in \trianglek} \left\{ \Omega(\iotak(\q)) - \langle \st, \iotak(\q) \rangle \right\}.
    \end{align*}
    This  optimization problem is identical to the one defining $\argmaxok(\stk)$, proving the result.
\end{proof}

We now turn to the proof that the Gibbs update distributions match the KL-projection of the target onto the local simplex.

\begin{proof}
    \textbf{Conditional distribution.}
    The conditional distribution is defined by renormalizing the global mass restricted to $\Yk$:
    \begin{align}
    \label{eq:cond_def}
        \argmaxo(\st)\left(\y \mid \yk_{[\Sk]}\right) = \frac{\argmaxo(\st)(\y)}{\sum_{\y' \in \Yk} \argmaxo(\st)(\y')} \quad \forall \y \in \Yk.
    \end{align}
    
    \textbf{KL projection.}
    Let $\piv \triangleq \argmaxo(\st)$ and $\q\in\trianglek$. We have:
    \begin{align*}
        \KL(\iotak(\q) \| \piv) = \sum_{\y\in\Yk}q_\y\log\left(\frac{q_\y}{\pi_\y}\right).
    \end{align*}
    We define the Lagrangian with a multiplier $\lambda \in \RR$ for the constraint $\sum_{\y \in \Yk} q_\y = 1$:
    \begin{align*}
        \cL(\q, \lambda) = \sum_{\y \in \Yk} q_\y \log (q_\y) - \sum_{\y \in \Yk} q_\y \log (\pi_\y) + \lambda \left( \sum_{\y \in \Yk} q_\y - 1 \right).
    \end{align*}
    
    \textbf{Optimality conditions.}
    The first-order condition with respect to any coordinate $q_\y$ is:
    \begin{align*}
        \frac{\partial \cL}{\partial q_\y}(\q^\star,\lambda^\star) = 1 + \log q_\y^\star - \log (\pi_\y) + \lambda^\star = 0.
    \end{align*}
    Solving for $q_\y^\star$:
    \begin{align*}
        q_\y^\star = \pi_\y \exp(-(1+\lambda^\star)).
    \end{align*}
    Summing over $\y \in \Yk$ to satisfy the simplex constraint:
    \begin{align*}
        \sum_{\y \in \Yk} q_\y^\star = \exp(-(1+\lambda^\star)) \sum_{\y \in \Yk} \pi_\y = 1.
    \end{align*}
    This determines the Lagrange multiplier term: $\exp(-(1+\lambda^\star)) = (\sum_{\y' \in \Yk} \pi_{\y'})^{-1}$. Substituting this back into the expression for $q_\y^\star$, we obtain:
    \begin{align*}
        q_\y^\star = \frac{\pi_\y}{\sum_{\y' \in \Yk} \pi_{\y'}}.
    \end{align*}
    This matches the definition of the conditional distribution in \cref{eq:cond_def} exactly.
\end{proof}

\subsection{Proof of \cref{prop:one_step_grad} (\nameref{prop:one_step_grad})}
\label{proof:one_step_grad}

\begin{proof}
    We first characterize the domain of $\Omega_\y$. We can write $F_\y$ as a convex combination of local maximum functions: $F_\y(\s) = \sum_{S \in \supp(\nu)} \nu(S) f_S(\s)$, where $f_S(\s) \triangleq f_{\Omega^{\cY_S}}(\s_{[\cY_S]})$.

    To find the domain of the conjugate $f_S^*$, let $h_S(\u) \triangleq f_{\Omega^{\cY_S}}(\u)$ for $\u \in \RR^{|\cY_S|}$. By definition of the smoothed maximum operator, we have $h_S=(\Omega^{\cY_S})^*$. Since $\Omega^{\cY_S}$ is strictly convex, proper, and lower semi-continuous on the compact set $\triangle^{\cY_S}$, the Fenchel-Moreau theorem yields $h_S^*=(\Omega^{\cY_S})^{**}=\Omega^{\cY_S}$. Thus, $\dom(h_S^*)=\dom(\Omega^{\cY_S}) = \triangle^{\cY_S}$.

    For any global score vector $\s \in \RRY$, we decompose it into components inside the local neighborhood, $\s_{\mathrm{in}} \triangleq \s_{[\cY_S]}$, and components outside, $\s_{\mathrm{out}} \triangleq \s_{[\cY \setminus \cY_S]}$. We can then write $f_S(\s) = h_S(\s_{\mathrm{in}})$. By definition of the Fenchel conjugate:
    \begin{align*}
        f_S^*(\p) &= \sup_{\s_{\mathrm{in}}, \s_{\mathrm{out}}} \left\{ \langle \p_{\mathrm{in}}, \s_{\mathrm{in}} \rangle + \langle \p_{\mathrm{out}}, \s_{\mathrm{out}} \rangle - h_S(\s_{\mathrm{in}}) \right\} \\
        &= \sup_{\s_{\mathrm{in}} \in \RR^{|\cY_S|}} \left\{ \langle \p_{\mathrm{in}}, \s_{\mathrm{in}} \rangle - h_S(\s_{\mathrm{in}}) \right\} + \sup_{\s_{\mathrm{out}} \in \RR^{|\cY \setminus \cY_S|}} \left\{ \langle \p_{\mathrm{out}}, \s_{\mathrm{out}} \rangle \right\} \\
        &= h_S^*(\p_{\mathrm{in}}) + I_{\{\0\}}(\p_{\mathrm{out}}),
    \end{align*}
    where $I_{\{\0\}}$ is $0$ at $\0$ and $+\infty$ elsewhere. For $f_S^*(\p)$ to be finite, we require $\p_{\mathrm{out}} = \0$ and $\p_{\mathrm{in}} \in \dom(h_S^*) = \triangle^{\cY_S}$. This means $\p$ must be a zero-padded distribution from the local simplex, so $\dom(f_S^*) = \iota_{\cY_S}(\triangle^{\cY_S})$.

    We now determine the domain of $\Omega_\y = (F_\y)^*$. Recall that $F_\y$ is a sum of scaled functions:
    $$ F_\y(\s) = \sum_{S \in \supp(\nu)} \nu(S) f_S(\s). $$

    Let $g_S(\s) \triangleq \nu(S) f_S(\s)$. A known property of the Fenchel conjugate then gives that for any $\nu(S) > 0$, we have:
    $$ g_S^*(\p) = \nu(S) f_S^*\left(\frac{\p}{\nu(S)}\right). $$

    Consequently, the effective domain of the scaled conjugate is scaled by the same factor:
    $$ \p \in \dom(g_S^*) \iff \frac{\p}{\nu(S)} \in \dom(f_S^*) \iff \p \in \nu(S) \dom(f_S^*). $$
    Thus, $\dom(g_S^*) = \nu(S) \dom(f_S^*)$.

    The Fenchel conjugate of a sum of functions is the inf-convolution (denoted by $\square$) of their individual conjugates \citep{rockafellar_convex_1970}:
    $$ \Omega_\y = (F_\y)^* = \square_{S \in \supp(\nu)} g_S^*. $$

    Since the effective domain of the inf-convolution is the Minkowski sum of the effective domains, we then have:
    \begin{align*}
        \dom(\Omega_\y) &= \sum_{S \in \supp(\nu)} \dom(g_S^*) \\
        &= \sum_{S \in \supp(\nu)} \nu(S) \dom(f_S^*) \\
        &= \sum_{S \in \supp(\nu)} \nu(S) \iota_{\cY_S}(\triangle^{\cY_S}).
    \end{align*}

    Because $\y$ is the reference solution defining the local neighborhoods, it satisfies $\y \in \cY_S$ for every subset $S$. Thus, we have $\delta_\y^{\cY_S}\in\triangle^{\cY_S}$, and by padding, we get $\delta^\cY_\y=\iota_{\cY_S}(\delta_\y^{\cY_S})\in\iota_{\cY_S}(\triangle^{\cY_S})$. Picking this specific point from each set in the Minkowski sum yields:
    \begin{align*}
        \sum_{S \in \supp(\nu)} \nu(S) \delta^\cY_\y = \left( \sum_{S \in \supp(\nu)} \nu(S) \right)\cdot \delta^\cY_\y =\delta^\cY_\y.
    \end{align*}
    This proves that $\delta^\cY_\y \in \dom(\Omega_\y)$, ensuring the target-dependent Fenchel-Young loss is well-defined and finite at the ground-truth distribution.

    Now, assuming that $\nu$ does not depend on $\s$, we can exchange the expectation and the gradient signs. The gradient of $F_\y$ is then given by Danskin's theorem and the chain rule:
    \begin{align*}
        \nabla_\s F_\y(\s) = \EE_{S\sim\nu}\left[ \iota_{\cY_S}\left(\p_{\Omega^{\cY_S}} \left(\s_{\left[\cY_S\right]}\right)\right)\right].
    \end{align*}

    The gradient of the target-dependent loss $L_{\Omega_\y}$ with respect to $\thetav$ at $\p = \delta^\cY_\y$ is:
    \begin{align*}
        \nabla_\thetav L_{\Omega_\y}(\thetav\,;\delta^\cY_\y) = \nabla_\thetav F_\y(\st) - \y = \EE_{S\sim\nu}\left[ \EE_{Y\sim\p_{\Omega^{\cY_S}} \left((\st)_{\left[\cY_S\right]}\right)}[Y]\right] - \y.
    \end{align*}

    This exactly matches the expected update of RLNS:
    \begin{align*}
        S\sim\nu\quad\text{and}\quad Y\mid S\sim\p_{\Omega^{\cY_S}}\left((\st)_{\left[\cY_S\right]}\right)\implies \EE_{S}[\EE[Y\mid S]] - \y = \nabla_\thetav L_{\Omega_\y}(\thetav\,;\delta^\cY_\y).
    \end{align*}
    Thus, RLNS sampling initialized at $\y$ and run for $1$ step produces an unbiased stochastic gradient.
\end{proof}

\subsection{Proof of \cref{prop:composite_likelihood} (\nameref{prop:composite_likelihood})}
\label{proof:composite_likelihood}

\begin{proof}
    In the Shannon entropy-regularized setting $\Omega=-\gamma H^s$, the smoothed local maximum evaluates to the log-sum-exp function:
    \begin{align*}
        f_{\Omega^{\cY_S}}\left(\s_{\left[\cY_S\right]}\right) &= \max_{\q\in\triangle^{\cY_S}}\langle \q,\s_{\left[\cY_S\right]}\rangle- \Omega\circ\iota_{\cY_S}(\q) \\
        &= \max_{\q\in\triangle^{\cY_S}}\langle \q,\s_{\left[\cY_S\right]}\rangle + \gamma H^s_{\cY_S}(\q) \quad \text{where } H^s_{\cY_S}(\q)\triangleq-\sum_{\y'\in\cY_S}q_{\y'}\log(q_{\y'})\\
        &= \gamma \log \sum_{\y'\in\cY_S}\exp\left(s_{\y'}/\gamma \right).
    \end{align*}
    
    We first determine the value of the additive constant $\Omega_\y(\delta_\y^\cY)$. By definition of the convex conjugate, we have:
    $$ \Omega_\y(\delta_\y^\cY) = \sup_{\s\in\RR^{|\cY|}} \left\{ \langle \s, \delta_\y^\cY \rangle - F_\y(\s) \right\} = \sup_{\s\in\RR^{|\cY|}} \left\{ s_\y - \EE_{S\sim\nu}\left[ \gamma\log\sum_{\y'\in\cY_S}\exp(s_{\y'}/\gamma) \right] \right\}. $$

    Notice that for any subset $S$, the ground-truth solution $\y$ always belongs to its own local neighborhood, meaning $\y \in \cY_S$. Because all terms in the exponential sum are strictly positive, the sum is lower-bounded by the single term corresponding to $\y$:
    $$ \sum_{\y'\in\cY_S}\exp(s_{\y'}/\gamma) \geq \exp(s_\y/\gamma). $$

    Taking the logarithm and multiplying by $\gamma > 0$, we get:
    $$ \gamma\log\sum_{\y'\in\cY_S}\exp(s_{\y'}/\gamma) \geq \gamma\log(\exp(s_\y/\gamma)) = s_\y. $$

    Therefore, for any score vector $\s \in \RR^{|\cY|}$ and any subset $S$, the difference inside the expectation is non-positive:
    $$ s_\y - \gamma\log\sum_{\y'\in\cY_S}\exp(s_{\y'}/\gamma) \leq 0. $$

    Taking the expectation over $S \sim \nu$ preserves this inequality, meaning $\Omega_\y(\delta_\y^\cY) \leq 0$. To show that the supremum is exactly $0$, we can construct a sequence of score vectors that asymptotically achieves it. Let $s_\y = 0$, and for all $\y' \neq \y$, let $s_{\y'} \to -\infty$. Under this limit, for any subset $S$:
    $$ \lim_{s_{\y'\neq \y} \to -\infty} \gamma\log\sum_{\y'\in\cY_S}\exp(s_{\y'}/\gamma) = \gamma\log(\exp(0) + 0) = 0. $$

    Consequently, the objective value tends to $0 - 0 = 0$. Since the supremum is bounded above by $0$ and can be asymptotically approached, we conclude that $\Omega_\y(\delta_\y^\cY) = 0$.

    Now, substituting $\st = (\langle\thetav,\y'\rangle)_{\y'\in\cY}$ and $\Omega_\y(\delta_\y^\cY) = 0$ into the loss function:
    \begin{align*}
        L_{\Omega_\y}(\thetav\,;\delta^\cY_\y) &= F_\y(\st) + \Omega_\y(\delta^\cY_\y) - \langle\thetav,\y\rangle\\
        &= \EE_{S\sim\nu}\left[ f_{\Omega^{\cY_S}}\left((\st)_{[\cY_S]}\right) \right] + 0 - \langle\thetav,\y\rangle\\
        &= \EE_{S\sim\nu}\left[ \gamma\log\sum_{\y'\in\cY_S}\exp(\langle\thetav,\y'\rangle/\gamma)  - \langle\thetav,\y\rangle\right].
    \end{align*}
    
    Moreover, denoting $\pi_\thetav\triangleq\argmaxo(\st)$, we know that $\pi_\thetav(\y)\propto\exp(\langle\thetav,\y\rangle/\gamma)$ in this setting. Thus, the conditional probability is:
    \begin{align*}
        \pi_\thetav(\y\mid \y_{[S]}) &= \frac{\pi_\thetav(\y)}{\sum_{\y'\in\cY_S}\pi_\thetav(\y')}\\
        &= \frac{\exp(\langle\thetav,\y\rangle/\gamma)}{\sum_{\y''\in\cY}\exp(\langle\thetav,\y''\rangle/\gamma)}\cdot\left( \frac{\sum_{\y'\in\cY_S}\exp(\langle\thetav,\y'\rangle/\gamma)}{\sum_{\y''\in\cY}\exp(\langle\thetav,\y''\rangle/\gamma)} \right)^{-1}\\
        &= \frac{\exp(\langle\thetav,\y\rangle/\gamma)}{\sum_{\y'\in\cY_S}\exp(\langle\thetav,\y'\rangle/\gamma)}.
    \end{align*}
    
    Taking the negative logarithm and multiplying by $\gamma$:
    $$ -\gamma \log \pi_\thetav(\y\mid \y_{[S]}) = \gamma\log\sum_{\y'\in\cY_S}\exp(\langle\thetav,\y'\rangle/\gamma) - \langle\thetav,\y\rangle. $$

    Substituting this back into the loss expression gives the final result:
    $$ L_{\Omega_\y}(\thetav\,;\delta^\cY_\y) = \EE_{S\sim\nu}\left[ -\gamma\log \pi_\thetav(\y\mid\y_{[S]})\right]. $$
\end{proof}

\subsection{Proof of \cref{prop:bias_bound} (\nameref{prop:bias_bound})}
\label{proof:bias_bound}

\begin{proof}
    This result follows from the perturbation bounds for uniformly ergodic Markov chains established by \citet{mitrophanov_sensitivity_2005}. Specifically, we apply Theorem 3.2 and Corollary 3.2 from that work.
    
    Let $\pi^*$ be the stationary distribution of $P^*$. \citet{mitrophanov_sensitivity_2005} shows that if $\tau_m(P^*) < 1$, then:
    \begin{align*}
        \TV\left( \piLNS\,, \pi^* \right) \leq \frac{\max_{\y \in \cY} \TV\left( P^m(\y\,, \cdot)\,, (P^*)^m(\y\,, \cdot) \right)}{1 - \tau_m(P^*)}
    \end{align*}
    (note that the original theorem is stated in terms of operator norms, which corresponds to exactly twice the TV distance for both the numerator and the left-hand side: the factors of $2$ cancel out).
    
    Furthermore, as shown in the proof of Corollary 3.2 in \citet{mitrophanov_sensitivity_2005}, the error over $m$ steps is bounded by $m$ times the single-step error:
    \begin{align*}
        \max_{\y \in \cY} \TV\left( P^m(\y\,,\cdot)\,,(P^*)^m(\y\,, \cdot)\right) \leq m \cdot\max_{\y \in \cY} \TV\left( P(\y\,, \cdot)\,, P^*(\y\,, \cdot) \right).
    \end{align*}
    
    We now bound the distance between the transition kernels. Let $\nu$ denote the probability distribution used to select the subset of coordinates $\Sk$ in Step 3 of \cref{algo:rlns} (and identically for the ideal Gibbs sampler). For a fixed current state $\yk \in \cY$, the transition kernel $P(\yk,\cdot)$ is a mixture of the local update distributions over all possible choices of $\Sk$. Using the inclusion map $\iota_{\Yk}$ from \cref{def:local_reg} to map local distributions to $\triangleY$, we denote:
    \begin{align*}
        \qlns &\triangleq \iota_{\Yk}\left(\argmaxok\left(\stk\right) \right)\in\triangleY,\\
        \qgibbs &\triangleq \iota_{\Yk}\left(\argmaxo\left(\st\right)\left(\cdot \mid \yk_{[\Sk]}\right)\right)\in\triangleY.
    \end{align*}
     Then, we have:
    \begin{align*}
        P(\yk,\cdot) = \sum_{\Sk \subseteq [d]} \nu(\Sk) \, \qlns.
    \end{align*}
    Similarly, the ideal Gibbs kernel $P^*(\yk\,, \cdot)$ uses the conditional distributions:
    \begin{align*}
        P^*(\yk,\cdot) = \sum_{\Sk \subseteq [d]} \nu(\Sk) \, \qgibbs.
    \end{align*}
    
    We compute the total variation distance between these two mixtures. Using the convexity of the TV distance, we have:
    \begin{align*}
        \TV\big( P(\yk,\cdot), P^*(\yk,\cdot) \big) &= \TV\Big( \sum_{\Sk} \nu(\Sk) \qlns, \sum_{\Sk} \nu(\Sk) \qgibbs\Big) \\
        &\leq \sum_{\Sk \subseteq [d]} \nu(\Sk) \, \TV\big(\qlns, \qgibbs\big).
    \end{align*}
    Since we have $\TV(\p,\p')=\TV(\iotak(\p),\iotak(\p'))$ for any $\p,\p'\in\trianglek$, \cref{def:consistency_gap} then gives that this local distance is bounded by $\gap$ for every $\yk$ and $\Sk$:
    \begin{align*}
        \TV\big(\qlns, \qgibbs \big) &\leq \max_{\y',S'}\TV\big(\q^{\mathrm{LNS}}(\y', S'), \q^{\mathrm{Gibbs}}(\y', S')\big) = \gap.
    \end{align*}
    Substituting this into the sum:
    \begin{align*}
        \TV\left(P(\yk,\cdot), P^*(\yk, \cdot) \right) &\leq \sum_{\Sk \subseteq [d]} \nu(\Sk) \, \gap \\
        &= \gap.
    \end{align*}
    Taking the maximum over $\yk \in \cY$ completes the proof.
\end{proof}

\subsection{Proof of \cref{prop:consistency_bound_smoothness} (\nameref{prop:consistency_bound_smoothness})}
\label{proof:consistency_bound_smoothness}

\begin{proof}
    Let $\piv \triangleq \argmaxo(\st)$ be the global regularized distribution. Fix a neighborhood $\Yk$ defined by some $\yk\in\cY$, and $\Sk\subseteq[d]$.
    Let $\qh,\q^* \in \trianglek$ denote the RLNS update and the ideal Gibbs update distributions (the projections of $\piv$ onto $\trianglek$ via $B_\Omega$ and $B_{-\gamma H^s}$, per \cref{prop:updates_as_projections}).

    \paragraph{Optimality conditions of the projection problems.}
    Both $\qh$ and $\q^*$ satisfy the first-order optimality conditions for their respective projection problems. For any $\bm{z} \in \trianglek$:
    \begin{align*}
        \langle \nabla_{\qh} B_\Omega\left(\iotak(\qh)\|\piv\right), \z-\qh \rangle &\geq 0,\\
        \langle \nabla_{\q^*} B_{-\gamma H^s}\left(\iotak(\q^*)\| \piv\right), \z-\q^* \rangle &\geq 0.
    \end{align*}
    Since for any $f$ we have $\nabla_1 B_f(\u\|\v) = \nabla f(\u) - \nabla f(\v)$, applying the chain rule gives:
    \begin{align*}
        \left\langle [\nabla\iotak (\qh)]^\top \left( \nabla\Omega(\iotak(\qh)) - \nabla\Omega(\piv)\right), \z-\qh \right\rangle &\geq 0,\\
        \left\langle [\nabla\iotak (\q^*)]^\top \left(\nabla(-\gamma H^s)(\iotak(\q^*)) - \nabla(-\gamma H^s)(\piv)\right), \z-\q^* \right\rangle &\geq 0.
    \end{align*}
    Since $\iotak$ is a linear map, it is its own Jacobian. Using the identity $\langle \A^\top \x, \y \rangle = \langle \x, \A\y \rangle$, we get:
    \begin{align}
        \label{eq:var_q_hat}
        \langle \nabla\Omega(\iotak(\qh)) - \nabla\Omega(\piv), \iotak(\bm{z} - \qh) \rangle &\geq 0,\\
        \label{eq:var_q_star}
        \langle \nabla(-\gamma H^s)(\iotak(\q^*)) - \nabla(-\gamma H^s)(\piv), \iotak(\bm{z} - \q^*) \rangle &\geq 0.
    \end{align}

    \paragraph{Decomposition.}
    We now decompose the gradient of $\Omega$ as $\nabla \Omega = \nabla(-\gamma H^s) + \nabla h_\gamma$. Substituting this into \cref{eq:var_q_hat} and applying it to $\bm{z}=\q^*\in\trianglek$, we get:
    \begin{align}
    \label{eq:helper_1}
        \langle \nabla(-\gamma H^s)(\iotak(\qh)) + \nabla h_\gamma(\iotak(\qh)) - \nabla(-\gamma H^s)(\piv) - \nabla h_\gamma(\piv), \iotak(\q^* - \qh) \rangle \ge 0.
    \end{align}
    Next, we apply \cref{eq:var_q_star} to $\bm{z}=\qh\in\trianglek$:
    \begin{align}
    \label{eq:helper_2}
        \langle \nabla(-\gamma H^s)(\iotak(\q^*)) - \nabla(-\gamma H^s)(\piv), \iotak(\qh - \q^*) \rangle \ge 0.
    \end{align}
    Summing \cref{eq:helper_1,eq:helper_2} and recognizing that $\iotak(\q^* - \qh) = - \iotak(\qh - \q^*)$, the $\nabla(-\gamma H^s)(\piv)$ terms cancel out, and we obtain:
    \begin{align*}
        \langle \nabla(-\gamma H^s)(\iotak(\qh)) - \nabla(-\gamma H^s)(\iotak(\q^*)), \iotak(\q^* - \qh) \rangle & \\
        + \langle \nabla h_\gamma(\iotak(\qh)) - \nabla h_\gamma(\piv), \iotak(\q^* - \qh) \rangle &\geq 0.
    \end{align*}
    Rearranging the terms and using the linearity of $\iotak$ yields:
    \begin{align}
    \label{eq:helper_4}
        &\left\langle \nabla h_\gamma(\iotak(\qh)) - \nabla h_\gamma(\piv), \iotak(\q^* - \qh) \right\rangle\notag\\
        \geq & \left\langle \nabla(-\gamma H^s)(\iotak(\qh)) - \nabla(-\gamma H^s)(\iotak(\q^*)), \iotak(\qh) - \iotak(\q^*) \right\rangle.
    \end{align}
    \paragraph{Strong convexity of $-\gamma H^s$.}
    By Pinsker's inequality, $-H^s$ is $1$-strongly convex with respect to the $\ell_1$ norm, which implies that $-\gamma H^s$ is $\gamma$-strongly convex with respect to the $\ell_1$ norm. Thus, for any $\u, \v \in \triangleY$, we have:
    \begin{align}
    \label{eq:cocoercivity}
         \langle \nabla(-\gamma H^s)(\u) - \nabla(-\gamma H^s)(\v)\,,\, \u - \v \rangle
         &\geq \gamma \|\u - \v\|_1^2.
    \end{align}
    The right-hand side of \cref{eq:helper_4} is now bounded by the strong convexity of $-\gamma H^s$, and \cref{eq:cocoercivity} gives:
    \begin{align*}
        \langle \nabla(-\gamma H^s)(\iotak(\qh)) - \nabla(-\gamma H^s)(\iotak(\q^*)), \iotak(\qh) - \iotak(\q^*) \rangle
        &\geq \gamma \|\iotak(\qh) - \iotak(\q^*)\|_1^2\\&\geq \gamma \|\qh - \q^*\|_1^2.
    \end{align*}
    Combining this with \cref{eq:helper_4}, we get:
    \begin{align}
    \label{eq:helper_3}
        \gamma \|\qh - \q^*\|_1^2 \leq \langle \nabla h_\gamma(\iotak(\qh)) - \nabla h_\gamma(\piv), \iotak(\q^* - \qh) \rangle.
    \end{align}
    
    \paragraph{Smoothness of $h_\gamma$.}
    Applying Hölder's inequality to the right-hand side of \cref{eq:helper_3}:
    \begin{align*}
        \gamma \|\qh - \q^*\|_1^2 &\le \| \nabla h_\gamma(\iotak(\qh)) - \nabla h_\gamma(\piv) \|_\infty \cdot \|\iotak(\q^* - \qh)\|_1 \\
        &\leq \| \nabla h_\gamma(\iotak(\qh)) - \nabla h_\gamma(\piv) \|_\infty \cdot \|\qh - \q^*\|_1.
    \end{align*}
    Assuming $\qh \neq \q^*$, dividing by $\|\qh - \q^*\|_1$, and using the definition $\TV(\p,\q) = \frac{1}{2}\|\p-\q\|_1$ yields:
    
    \begin{align*}
        \TV(\qh\,, \q^*) \leq \frac{1}{2\gamma} \| \nabla h_\gamma(\iotak(\qh)) - \nabla h_\gamma(\piv) \|_\infty .
    \end{align*}

    Then, from the $L_{\Omega,\gamma}$-smoothness of $h_\gamma$ for the $\ell_1$ norm (whose dual is the $\ell_\infty$ norm), we get:

    \begin{align*}
        \TV(\qh\,, \q^*) \leq \frac{L_{\Omega,\gamma}}{2\gamma} \| \iotak(\qh) - \piv \|_1 .
    \end{align*}
    Since $\piv,\iotak(\qh)\in\triangleY$, we conclude by using the fact that the diameter of $\triangleY$ in $\ell_1$ norm is $2$, which finally gives:
    \begin{align*}
        \TV(\qh\,, \q^*) \leq \frac{L_{\Omega,\gamma}}{\gamma}.
    \end{align*}
    Since this bound is independent of the chosen $\yk$ and $\Sk$ (which determine $\qh$ and $\q^*$ by determining $\Yk$), we finally get:
    \begin{align*}
        \gap \triangleq \max_{\yk,\Sk}\TV(\qh,\q^*) \leq \frac{L_{\Omega,\gamma}}{\gamma}.
    \end{align*}
\end{proof}

\subsection{Proof of \cref{prop:consistency_bound_perturbation} (\nameref{prop:consistency_bound_perturbation})}
\label{proof:consistency_bound_perturbation}

\begin{proof}
    Consider a fixed current solution $\yk \in \cY$ and a subset of coordinates $\Sk \subseteq [d]$, defining the local neighborhood $\Yk$.
    
    Let $\qh \triangleq \argmaxoek(\stk)$ denote the local regularized distribution, $\pi^* \triangleq \argmaxoe(\st)$ denote the global regularized distribution, and $\q^*\triangleq \argmaxoe(\st)(\cdot |\yk_{[\Sk]})$ denote the global regularized distribution conditioned on $\Yk$. We rely on the coupling induced by the shared noise variable $Z$.
    
    Let $Z \sim \cD$ be a random noise vector. We define two coupled random variables taking values in $\cY$ and $\Yk$ respectively:
    \begin{align*}
        Y^* &\triangleq \argmax_{\y \in \cY} \{ \langle \thetav + \varepsilon Z, \y \rangle \}, \\
        \hat{Y} &\triangleq \argmax_{\y \in \Yk} \{ \langle \thetav + \varepsilon Z, \y \rangle \}.
    \end{align*}
    By definition of the perturbation-based regularizer $\Omegae$ (\cref{def:perturbation_regularization}) and \cref{prop:perturbation_local}, the distributions of these variables correspond exactly to the global and local regularized distributions:
    \begin{align*}
        Y^* \sim \pi^* \quad \text{and} \quad \hat{Y} \sim \qh.
    \end{align*}
    
    \paragraph{Coupling event.}
    Let $E$ be the event that the global perturbed maximizer falls within the local neighborhood:
    \begin{align*}
        E \triangleq \{ Y^* \in \Yk \}.
    \end{align*}
    Since $\Yk \subseteq \cY$, if the maximum over the larger set $\cY$ belongs to the subset $\Yk$, it must strictly coincide with the maximum over the subset $\Yk$ (assuming unique maximizers almost surely, which holds if $\cD$ has a density). Therefore:
    \begin{align}
    \label{eq:coupling_implication}
        \forall \omega \in E, \quad Y^*(\omega) = \hat{Y}(\omega).
    \end{align}
    We denote the probability of this event by $\alpha \triangleq \PP(E) = \pi^*(\Yk)$.
    
    \paragraph{Decomposition of the local distribution.}
    The ideal conditional distribution $\q^*$ corresponds to the distribution of $Y^*$ conditioned on $E$:
    \begin{align*}
        \forall \y \in \Yk, \quad q^*_\y = \PP(Y^* = \y \mid E).
    \end{align*}
    We decompose the probability mass function of $\hat{Y}$ (which is $\qh$) by conditioning on $E$:
    \begin{align*}
        \hat{q}_\y = \PP(\hat{Y} = \y)= \PP(\hat{Y} = \y \mid E)\PP(E) + \PP(\hat{Y} = \y \mid \overline{E})\PP(\overline{E}).
    \end{align*}
    Using the coupling implication from \cref{eq:coupling_implication}, we have $\hat{Y} = Y^*$ on $E$. Thus:
    \begin{align*}
        \PP(\hat{Y} = \y \mid E) = \PP(Y^* = \y \mid E) = q^*_\y.
    \end{align*}
    Substituting this back, and letting $\r \in \trianglek$ be defined by $r_\y \triangleq \PP(\hat{Y} = \y \mid \overline{E})$, we obtain the mixture:
    \begin{align*}
        \hat{q}_\y = \alpha q^*_\y + (1-\alpha) r_\y.
    \end{align*}
    
    \paragraph{Total variation bound.}
    We now compute the total variation distance between $\qh$ and $\q^*$:
    \begin{align*}
        \TV(\qh, \q^*) &= \frac{1}{2} \sum_{\y \in \Yk} | \hat{q}_\y - q^*_\y | \\
        &= \frac{1}{2} \sum_{\y \in \Yk} | \alpha q^*_\y + (1-\alpha) r_\y - q^*_\y | \\
        &= \frac{1}{2} \sum_{\y \in \Yk} | (1-\alpha)(r_\y - q^*_\y) | \\
        &= (1-\alpha) \left( \frac{1}{2} \sum_{\y \in \Yk} | r_\y - q^*_\y | \right)\\
        &= (1-\alpha)\cdot \TV(\r, \q^*).
    \end{align*}
    Since the total variation distance is bounded by $1$, we then have:
    \begin{align*}
        \TV(\qh, \q^*) \leq 1 - \alpha = 1 - \pi^*(\Yk).
    \end{align*}
    
    \paragraph{Conclusion.}
    The local consistency gap $\gape$ is defined as the maximum of this total variation distance over all choices of $\yk$ and $\Sk$ (\cref{def:consistency_gap}). Taking the maximum yields:
    \begin{align*}
        \gape \leq 1 - \min_{\yk, \Sk} \pi^*(\Yk) = \max_{\yk, \Sk} \big(1 - \pi^*(\Yk)\big) = \max_{\yk, \Sk} \argmaxoe(\st)(\cY\setminus\Yk).
    \end{align*}
\end{proof}

\subsection{Proof of \cref{prop:product_space_exact} (\nameref{prop:product_space_exact})}
\label{proof:product_space_exact}

\begin{proof}
    \textbf{Factorization of the global distribution.} 
    Recall that the global regularized distribution is defined by the law of the perturbed maximizer $Y^* \triangleq \argmax_{\y\in\cY} \langle\thetav + \varepsilon Z, \y\rangle$ with $Z \sim \cD$.
    Because the feasible set is a product space $\cY = \prod_{m=1}^M \cY_m$, the dot product additively decomposes over the blocks $B_m$. The global maximization therefore separates into $M$ independent local maximizations:
    \begin{align*}
        Y^* = \left( \argmax_{\y_m \in \cY_m} \langle \thetav_{[B_m]} + \varepsilon Z_{[B_m]}, \y_m \rangle \right)_{m=1}^M.
    \end{align*}
    Since the noise blocks $(Z_{[B_m]})_{m=1}^M$ are mutually independent, the components $(Y^*_{[B_m]})_{m=1}^M$ are independent random variables. Thus, the global regularized distribution $\pi^* \triangleq \argmaxoe(\st)$ factorizes over the blocks:
    \begin{align*}
        \pi^*(\y) = \prod_{m=1}^M \pi^*_m(\y_{[B_m]}), \quad \text{where} \quad \pi^*_m(\y_m) \triangleq \PP\left( \argmax_{\y'_{m} \in \cY_m} \langle \thetav_{[B_m]} + \varepsilon Z_{[B_m]}, \y'_{m} \rangle = \y_m \right).
    \end{align*}

    \textbf{Conditional distribution.}
    Let $\yk \in \cY$ be a current solution, and $\Sk$ be a union of blocks defining the frozen variables. Let $\Fk \triangleq [d] \setminus \Sk$ be the complement (the free variables), which is also a union of blocks. Due to the independence of the blocks, conditioning the global distribution $\pi^*$ on the frozen variables $\yk_{[\Sk]}$ simply restricts the product to the free blocks:
    \begin{align}
    \label{eq:product_conditional}
        \pi^*\left( \y \mid \yk_{[\Sk]} \right) = \prod_{B_m \subseteq \Fk} \pi^*_m(\y_{[B_m]}) \quad \text{for any } \y \in \Yk,
    \end{align}
    where $\Yk = \{ \y \in \cY \mid \y_{[\Sk]} = \yk_{[\Sk]} \}$.

    \textbf{Local RLNS update.}
    By \cref{prop:perturbation_local}, the local perturbed RLNS update samples from $\argmaxoek(\stk)$ by computing $\hat{Y} \triangleq \argmax_{\y \in \Yk} \langle \thetav + \varepsilon Z, \y \rangle$. 
    Because $\y_{[\Sk]}$ is fixed to $\yk_{[\Sk]}$ for all $\y \in \Yk$, the maximization over $\Yk$ reduces to maximizing over the free blocks:
    \begin{align*}
        \hat{Y}_{[\Fk]} = \argmax_{\y_{\text{free}} \in \prod_{B_m \subseteq \Fk} \cY_m} \sum_{B_m \subseteq \Fk} \langle \thetav_{[B_m]} + \varepsilon Z_{[B_m]}, \y_{[B_m]} \rangle.
    \end{align*}
    This subproblem is structurally identical to the one defining the free components of the global maximizer $Y^*$. Since $Z\sim\cD$, the distribution of $\hat{Y}_{[\Fk]}$ is exactly $\prod_{B_m \subseteq \Fk} \pi^*_m$.
    
    Thus, the local regularized distribution matches the conditional global distribution from \cref{eq:product_conditional}:
    \begin{align*}
        \argmaxoek(\stk)(\y) = \pi^*\left( \y \mid \yk_{[\Sk]} \right).
    \end{align*}
    By \cref{def:consistency_gap}, this exact match implies that the local consistency gap is $\gape = 0$.
    Since the local updates are equivalent to conditional sampling from $\pi^*$, \cref{algo:perturbed_lns} precisely operates as a block Gibbs sampler targeting $\pi^*$, and its stationary distribution is exactly $\piLNSe = \argmaxoe(\st)$.
\end{proof}

\subsection{Proof of \cref{prop:cut_bound} (\nameref{prop:cut_bound})}
\label{proof:cut_bound}

\begin{proof}
    Consider a fixed current solution $\yk \in \cY$ and a subset of frozen coordinates $\Sk \subseteq [d]$, with free coordinates $\Fk = [d] \setminus \Sk$.
    Let $\qh \triangleq \argmaxoek(\stk) \in \trianglek$ denote the local RLNS update distribution, and $\q^* \triangleq \argmaxoe(\st)(\cdot \mid \yk_{[\Sk]}) \in \trianglek$ denote the ideal global conditional distribution. Let $A \triangleq \{ Y^*_{\text{relax}}(\Sk) \in \cY \}$, so that $\overline{A} = \{ Y^*_{\text{relax}}(\Sk) \in W_\partial(\Sk) \}$. Notice that if either $\PP(\overline{A} \mid (Y^*_{\text{relax}}(\Sk))_{[\Sk]} = \yk_{[\Sk]}) = 1$ or $\PP(\overline{A} \mid Y^*_{[\Sk]} = \yk_{[\Sk]}) = 1$, the bound holds trivially since $\TV(\qh, \q^*) \leq 1$; we thus assume both probabilities are strictly less than $1$.

    \textbf{1: Reference feasible relaxed distribution.}
    Define the relaxed local neighborhood $\Yk_{\text{relax}}$ by
    \begin{align*}
        \Yk_{\text{relax}} \triangleq \{ \y \in \cY_{\text{relax}}(\Sk) \mid \y_{[\Sk]} = \yk_{[\Sk]} \}
    \end{align*}
    and the coupled local maximizers $\hat{Y}$ and $\hat{Y}_{\text{relax}}$ by
    \begin{align*}
        \hat{Y} \triangleq \argmax_{\y \in \Yk} \langle \thetav + \varepsilon Z, \y \rangle \sim \qh \quad \text{and} \quad \hat{Y}_{\text{relax}} \triangleq \argmax_{\y \in \Yk_{\text{relax}}} \langle \thetav + \varepsilon Z, \y \rangle.
    \end{align*}
    Since $\cY_{\text{relax}}(\Sk) = \cY_F(\Sk) \times \cY_S(\Sk)$ and $Z_{[\Fk]}$ is independent of $Z_{[\Sk]}$, we have $(\hat{Y}_{\text{relax}})_{[\Fk]} = (Y^*_{\text{relax}}(\Sk))_{[\Fk]}$, which is independent of $(Y^*_{\text{relax}}(\Sk))_{[\Sk]}$.
    Consequently, letting $E_\partial \triangleq \{ \hat{Y}_{\text{relax}} \notin \Yk \}$, we have $\PP(E_\partial) = \PP(\overline{A} \mid (Y^*_{\text{relax}}(\Sk))_{[\Sk]} = \yk_{[\Sk]}) < 1$, and we can define a reference distribution $\r \in \trianglek$ by conditioning $\hat{Y}_{\text{relax}}$ on $\overline{E_\partial}$:
    \begin{align*}
        r_\y \triangleq \PP\big(\hat{Y}_{\text{relax}} = \y \;\big|\; \overline{E_\partial}\big) = \PP\big(Y^*_{\text{relax}}(\Sk) = \y \;\big|\; A,\; (Y^*_{\text{relax}}(\Sk))_{[\Sk]} = \yk_{[\Sk]}\big) \quad \forall \y \in \Yk.
    \end{align*}

    \textbf{2: Bounding the distance to the local RLNS update.}
    Since $\Yk \subseteq \Yk_{\text{relax}}$, we have that on the event $\overline{E_\partial} = \{ \hat{Y}_{\text{relax}} \in \Yk \}$ the maximizer over $\Yk_{\text{relax}}$ belongs to $\Yk$, so $\hat{Y} = \hat{Y}_{\text{relax}}$ almost surely. Conditioning $\qh(\y) = \PP(\hat{Y} = \y)$ on $\overline{E_\partial}$ and $E_\partial$ yields, for all $\y \in \Yk$:
    \begin{align*}
        \qh(\y) = (1 - \PP(E_\partial)) r_\y + \PP(E_\partial) \PP(\hat{Y} = \y \mid E_\partial),
    \end{align*}
    which implies $\TV(\qh, \r) \leq \PP(E_\partial) = \PP\big(Y^*_{\text{relax}}(\Sk) \in W_\partial(\Sk) \;\big|\; (Y^*_{\text{relax}}(\Sk))_{[\Sk]} = \yk_{[\Sk]}\big)$.

    \textbf{3: Bounding the distance to the ideal conditional.}
    Since $\cY \subseteq \cY_{\text{relax}}(\Sk)$, we have that on the event $A = \{ Y^*_{\text{relax}}(\Sk) \in \cY \}$ the equality $Y^* = Y^*_{\text{relax}}(\Sk)$ holds almost surely. Therefore we have that $A \cap \{ Y^*_{[\Sk]} = \yk_{[\Sk]} \} = A \cap \{ (Y^*_{\text{relax}}(\Sk))_{[\Sk]} = \yk_{[\Sk]} \}$, and for all $\y \in \Yk$, $\PP(Y^* = \y \mid A,\, Y^*_{[\Sk]} = \yk_{[\Sk]}) = r_\y$. Conditioning $\q^*(\y) = \PP(Y^* = \y \mid Y^*_{[\Sk]} = \yk_{[\Sk]})$ on $A$ and $\overline{A}$ therefore yields:
    \begin{align*}
        \q^*(\y) = \PP(A \mid Y^*_{[\Sk]} = \yk_{[\Sk]}) r_\y + \PP(\overline{A} \mid Y^*_{[\Sk]} = \yk_{[\Sk]}) \PP(Y^* = \y \mid \overline{A},\, Y^*_{[\Sk]} = \yk_{[\Sk]}),
    \end{align*}
    giving
    $\TV(\q^*, \r) \leq \PP(\overline{A} \mid Y^*_{[\Sk]} = \yk_{[\Sk]}) = \PP\big(Y^*_{\text{relax}}(\Sk) \in W_\partial(\Sk) \;\big|\; Y^*_{[\Sk]} = \yk_{[\Sk]}\big)$.

    \textbf{4: Triangle inequality and union bound.}
    Combining Steps 2 and 3 via
    \begin{align*}
        \TV(\qh, \q^*) \leq \TV(\qh, \r) + \TV(\r, \q^*)
    \end{align*}
    and taking the maximum over $\yk \in \cY$ and $\Sk$ yields the first bound on $\gape$. Applying a union bound over $c \in C_\partial(\Sk)$ to $W_\partial(\Sk) = \bigcup_{c \in C_\partial(\Sk)} W_c(\Sk)$ inside both conditional probabilities yields the second bound.
\end{proof}
\newpage

\section{Experimental details for $\kappa$-hot vector estimation}
\label{sec:appendix_k_hot_details}

This section details the experimental setup for the numerical results presented in \cref{sec:k-hot}. 

\paragraph{Problem formulation.} 
We consider the set of $\kappa$-hot vectors, $\cY\triangleq \{\y\in\{0,1\}^d\mid \sum_{i=1}^dy_i=\kappa\}$, with $d=1000$ and $\kappa=500$. The score vectors $\thetav \in \RR^d$ are sampled from a standard Normal distribution $\mathcal{N}(\0, \I_d)$. We run all algorithms for $10^5$ iterations and average the metrics over $128$ random seeds (controlling randomness of the generation of $\thetav$ as well as in the MCMC and Monte-Carlo processes) for statistical significance. To measure convergence, we compute the following mean absolute error (MAE):
\begin{align*}
    \text{MAE}^{(K)} = \frac{1}{d}\sum_{i=1}^d\left|\hat{\mu}^{(K)}_i-\mu_i\right|,
\end{align*}
where $\hat{\muv}^{(K)}=\frac{1}{K}\sum_{k=1}^K\yk$ is the $K$-th RLNS estimate of $\muv\triangleq\EE_{\argmaxo(\st)}[Y]$.
\paragraph{Variable selection.} 
For the neighborhood system in RLNS, at each step $k$, we uniformly sample a random mask to freeze a fixed fraction of the variables (e.g., $70\%$ frozen, $30\%$ free). The local subproblem then amounts to placing exactly $\kappa - l$ ones among the free variables, where $l$ is the number of ones already assigned in the frozen subset.

\paragraph{Regularization and Oracles.} 
We evaluate the methods under two regularization regimes:
\begin{itemize}[topsep=0pt,itemsep=1.5pt,parsep=2pt,leftmargin=15pt]
    \item \textbf{Shannon Regularization:} We set the temperature $\gamma=1.0$. The local Gibbs sampling $\argmaxok(\stk)$ is performed exactly using regularized dynamic programming on the free variables. The global target $\muv$ is analytically computed using the global regularized dynamic programming formulation \citep{vivier-ardisson_differentiable_2026}.
    \item \textbf{Perturbation-based Regularization:} We use standard Gaussian noise $\mathcal{N}(\0, \I_d)$ with scaling $\varepsilon=1.0$. Local sampling is performed via perturbed local maximization, which trivially reduces to identifying the top-$(\kappa-l)$ perturbed scores among the free variables. The global target $\muv$ is estimated to high precision via independent Monte-Carlo samples of the global perturbed MAP. 
\end{itemize}

\paragraph{Baselines.}
We compare RLNS against i.i.d.\ Monte-Carlo sampling, which repeatedly calls the global exact regularized oracles (i.e., global perturb-and-MAP \citep{papandreou_perturb-and-map_2011,berthet_learning_2020} for perturbation-based regularization and sampling from the global Gibbs distribution $\argmaxo(\st)(\y)\propto\exp(\langle\thetav,\y\rangle/\gamma)$ for Shannon regularization) and Local Search MCMC \citep{vivier-ardisson_learning_2025}. For Local Search MCMC, the proposal mechanism uses the same neighborhood system as RLNS, also uniformly sampling a random mask to freeze the same fixed fraction of variables. It then uniformly samples a new valid $\kappa$-hot vector (by sampling a $(\kappa-l)$-hot vector on the free coordinates), which is then accepted or rejected via the standard Metropolis-Hastings ratio.

Every Monte-Carlo and MCMC run is performed on a CPU with $64$GB of RAM and $16$ cores.

\section{Experimental details for the generalized assignment problem}
\label{sec:appendix_gap_details}

This section presents the experimental details for the numerical results presented in \cref{sec:gap}.

\paragraph{Instance generation.}
We consider GAP instances with $n=50$ items and $m=5$ bins. For each instance, item-bin weights $w_{ij}$ are sampled uniformly from the discrete set $\{1, \dots, 10\}$. To create tight constraints that force non-trivial combinatorial trade-offs, the capacity $C_j$ for all bins in a given instance is set identically to $C = \lfloor \frac{1}{2}\sum_{i=1}^n \sum_{j=1}^m w_{ij} / m^2 \rfloor$.

\paragraph{Dataset generation and expert policy.}
We generate a dataset of $1000$ independent GAP instances. For each instance, we generate latent profile vectors to construct relational features:
\begin{itemize}
    \item \textbf{Item features:} $\v_i \sim \cN(\0, \I)$ for $i \in [n]$.
    \item \textbf{Bin features:} $\u_j \sim \cN(\0, \I)$ for $j \in [m]$.
\end{itemize}
The edge-wise feature vector for an item-bin pair is the concatenation of their respective profiles, $\x_{ij} = [\v_i, \u_j] \in \RR^{20}$. To generate the ground-truth assignment labels $\y \in \{0, 1\}^{n \times m}$, a randomly initialized target neural network processes the features to generate true (unseen by the model) underlying scores $\theta^*_{ij}$, and the target assignment $\y=\hat{\y}(\thetav^*)$ is found using the \texttt{scipy.optimize.milp} global solver. A time limit of 60 seconds is enforced per instance during dataset generation.

\paragraph{Model architecture.}
We learn a mapping from the relational features $\x_{ij}$ to the predicted scores $\theta_{ij}$ using a shared multi-layer perceptron. The network operates independently on each item-bin pair and shares weights across all edges. The architecture consists of two hidden layers of size 32 with ReLU activations, followed by a linear projection to a scalar cost.

\paragraph{Local oracles and LNS subproblem.}
To scale learning, RLNS utilizes a randomized neighborhood system. At each iteration $k$, the assignments of all bins except one randomly selected free bin $j_k \in [m]$ are frozen. The LNS subproblem can be written as the following $0/1$ Knapsack problem:
\begin{align*}
    \max_{\y_{[\colon,j_k]}\in\{0,1\}^{n}}\sum_{i=1}^n \theta_{ij_k}y_{ij_k} \quad \text{s.t.} \quad
    \sum_{i=1}^n \tilde{w}_{ij_k}y_{ij_k}\leq C,
\end{align*}
where the weight $\tilde{w}_{ij_k}$ is set to $(C+1)$ if item $i$ is already assigned to another bin to prevent assigning it twice, and to $w_{ij_k}$ otherwise.

This subproblem admits an exact dynamic programming solution with $\cO(nC)$ complexity \citep{kellerer_knapsack_2004}. Thus, we can easily sample from $\argmaxoek(\stk)$ when using perturbation-based regularization by local perturbed maximization. When using Shannon regularization, one can sample from the local Gibbs distribution $\argmaxok(\stk)$ via regularized dynamic programming with the same complexity using \citet[Algorithm 3]{vivier-ardisson_differentiable_2026}.

\paragraph{Training details.} For \cref{fig:gap_iterations}, all models are trained for $10,000$ steps using the Adam optimizer with a learning rate of $2\times10^{-3}$ and a batch size of $32$, on a CPU with $64$GB of RAM and $16$ cores. For perturbation-based regularization (in both RLNS and PFY), we use standard Gaussian noise $\cD = \cN(\0, \I)$ with temperature $\varepsilon = 1.0$.

\paragraph{Gradient variance analysis.}
To robustly evaluate the stability of the RLNS stochastic gradients in \cref{fig:gap_time_variance} (top), we measure the trace of the gradient covariance matrix using repeated sampling. For these experiments, we fix a single GAP instance and sample a batch of $64$ logit vectors $\thetav$ from a standard Normal distribution. For each score vector, we draw $50$ independent gradient samples to compute the per-instance covariance trace. The reported variance metrics represent the average across this batch of $64$ queried score vectors. All error bars indicate the $95\%$ confidence intervals of this mean across the batch, calculated using the standard normal multiplier ($1.96$). Explicit random seeds are used to strictly control the generation of the underlying problem instance weights, the queried score vectors, and the internal stochasticity of the MCMC and perturbation sampling mechanisms.

\paragraph{Oracle execution time benchmarking.}
To ensure statistically significant time comparisons between the local dynamic programming oracle and the global MILP solver in \cref{fig:gap_time_variance} (bottom), we benchmark execution times across $100$ independent trials for each problem scale ($(n,m)=(100,10)$, $(n,m)=(200,20)$, and $(n,m)=(400,40)$). Within each scale, the underlying item weights and bin capacities are generated using a fixed random seed (following the same generation process as explained above). For each of the $100$ trials, we sample a new logit vector $\thetav$ and, for the local oracle, uniformly sample a new free bin to re-optimize. The reported execution times are the arithmetic means over the $100$ trials, with error bars representing the $95\%$ confidence intervals derived from a Student's $t$-distribution.

\section{Experimental details for stochastic vehicle scheduling}
\label{sec:appendix_svsp_details}

This section details the complete experimental setup, dataset generation procedure, model architecture, and hyperparameters for the stochastic vehicle scheduling problem (SVSP) experiments presented in \cref{sec:svsp}. While the problem formulation and dataset generation process are standard and identical to, e.g., \citet{parmentier_learning_2022,dalle_learning_2022,hoppe_structured_2025}, we detail them here for completeness.

\subsection{Problem formulation}
The problem is modeled on a directed acyclic graph $\cG=(\cV,\cE)$ with $|\cV|=n+2$ nodes (a source depot $0$, $n$ job nodes $[n]$, and a sink depot $n+1$) and $|\cE|$ edges. An edge $(i,j)\in\cE$ exists from source to every job, from every job to sink, and from job $i$ to job $j$ if the deterministic end time of task $i$ plus travel time allows reaching task $j$ before its start. The number of edges varies across instances, as it depends on the spatial and temporal configuration of the tasks.

A feasible routing $\y\in\{0,1\}^{|\cE|}$ satisfies flow conservation and unit demand constraints at every job node:
\begin{align*}
    \sum_{e\in\delta^-(v)} y_e = 1 \quad \text{and} \quad \sum_{e\in\delta^+(v)} y_e = 1 \quad \forall v\in[n],
\end{align*}
where $\delta^-(v)$ and $\delta^+(v)$ denote the sets of incoming and outgoing edges of $v$. Each vehicle corresponds to a path from source to sink, so the number of vehicles is $\sum_{e\in\delta^+(0)} y_e$.

Given a routing $\y$ and a scenario $\omega\in[S]$, delays propagate along each vehicle path via the recursion $R_{v}^{(\omega)}=\varepsilon_v^{(\omega)}+\max(R_u^{(\omega)}-\sigma_{uv}^{(\omega)},\,0)$, where $\varepsilon_v^{(\omega)}$ is the intrinsic delay at node $v$, $\sigma_{uv}^{(\omega)}$ is the slack (time margin) on edge $(u,v)$, and $(u,v)$ is the predecessor edge of $v$ in the routing. The total cost of routing $\y$ is the sum of the vehicle cost and the average total delay cost over the scenarios:
\begin{align*}
    c(\y) = c_{\mathrm{vehicle}}\cdot\sum_{e\in\delta^+(0)} y_e + c_{\mathrm{delay}}\cdot\frac{1}{S}\sum_{\omega=1}^S\sum_{v=1}^n R_v^{(\omega)}.
\end{align*}

\subsection{Dataset generation and expert policy}
We generate a dataset of $1000$ independent SVSP instances. For each instance, we sample a random city layout simulating spatial areas and transit delays. 
\begin{itemize}
    \item \textbf{Stochasticity:} Task locations, start times, and travel times are sampled randomly. Intrinsic task delays ($\varepsilon_v^{(\omega)}$) and regional inter-area travel delays follow LogNormal distributions, generating the $10$ distinct delay scenarios per instance.
    \item \textbf{Expert policy:} The ground-truth solution $\y$ for each instance is computed by solving a two-stage stochastic MILP that minimizes the expected total cost $c(\y)$ with full knowledge of all $S$ scenarios. To handle the bilinear delay-propagation term $y_{uv} R_u^{(\omega)}$ without requiring a big-M parameter, we use exact McCormick envelope linearizations. The MILP is solved using SciPy with a time limit of $120$ seconds per instance.
    \item \textbf{Size:} Each instance features $n=50$ job tasks and $S=10$ stochastic scenarios, which is the largest problem size handled by the expert solver. In terms of industrial applications, this scale corresponds to standard long-haul fleet aircraft routing problems.
\end{itemize}

\subsection{Features and model architecture}
For each edge $e \in \cE$ in a given graph, we construct a $20$-dimensional feature vector $\x_e \in \RR^{20}$. The features capture the deterministic and stochastic characteristics of the edge:
\begin{itemize}
    \item \textbf{Feature 1:} The deterministic travel time for the edge.
    \item \textbf{Feature 2:} The vehicle cost (equal to $c_{\mathrm{vehicle}}$ if the edge originates from the source depot, and $0$ otherwise).
    \item \textbf{Features 3--11:} The $9$ deciles of the slack (time margin) distribution $\sigma_{uv}^{(\omega)}$ computed across the $S=10$ scenarios.
    \item \textbf{Features 12--20:} The cumulative distribution function (CDF) of the slack distribution evaluated at $9$ predefined time thresholds: $\{-100, -50, -20, -10, 0, 10, 50, 200, 500\}$.
\end{itemize}

The predictive model uses a standard feed-forward architecture. It applies a shared two-layer perceptron with a hidden dimension of $64$ independently to each edge's feature vector $\x_e$ to output the predicted scalar cost $\theta_e$.

Given predicted edge costs $\thetav\in\RR^{|\cE|}$, the combinatorial optimization layer is then a flow linear program (LP) $\hat{\y}(\thetav)\triangleq\argmax_{\y\in\cY}\langle\thetav,\y\rangle$.

\subsection{Training details and hyperparameters}
All models are trained for $10,000$ steps using the Adam optimizer with a learning rate of $10^{-4}$ and a batch size of $32$. The dataset is split into $70\%$ training, $15\%$ validation, and $15\%$ testing. All training configurations are repeated across $10$ different random seeds controlling dataset and model initialization randomness for statistical significance. Each training is performed on a CPU with $64$GB of RAM and $16$ cores.

\paragraph{RLNS-specific parameters.}
Our proposed method is parameterized as follows:
\begin{itemize}
    \item \textbf{Iterations:} We run $K=1$ RLNS iteration per training step to compute the gradient estimator.
    \item \textbf{Regularization:} We use perturbation-based regularization with $\cN(\0, \I)$ standard Gaussian noise and a temperature scaling parameter $\varepsilon=1.0$.
    \item \textbf{Initialization:} RLNS is initialized directly at the ground-truth expert solution.
    \item \textbf{Variable selection scheme:} We use a route-based mask generation heuristic. At each step, we identify the distinct vehicle routes present in the current routing assignment. We select $66\%$ of these routes uniformly at random and freeze all variables corresponding to the tasks visited by those routes.
\end{itemize}

\paragraph{Baseline parameters.}
The baselines share the identical underlying model architecture, optimization hyperparameters, and evaluation metrics:
\begin{itemize}
    \item \textbf{Global perturbed Fenchel-Young loss (PFY) \citep{berthet_learning_2020}:} Uses the same perturbation-based regularization setup (Gaussian noise, $\varepsilon=1.0$), but calls the exact global MAP oracle over the full graph $\cG$ at every training step.
    \item \textbf{Blackbox differentiation (DBB) \citep{vlastelica_differentiation_2020}:} Continuous interpolation of the solver. We performed a hyperparameter sweep over the interpolation strength $\lambda \in \{10^{-1}, 5\times10^{-1}, 1, 5, 10^1, 10^2, 5\times10^2, 10^3, 5\times10^3, 10^4\}$ and report the best performing configuration, which was $\lambda=10^2$.
    \item \textbf{Binary cross-entropy (BCE):} An unstructured binary cross-entropy loss applied directly to the edge-level predictions $\thetav$ against the expert binary labels $\y$. See, e.g., \citet{joshi_efficient_2019} for a similar approach on a routing combinatorial optimization problem.
    \item \textbf{Greedy policy:} A non-learned baseline that solves a standard minimum-cost flow problem on the DAG, placing negative costs on the source edges to minimize the total number of vehicles, entirely ignoring the stochastic delays.
\end{itemize}

\end{document}